%% file: MRI.tex
\documentclass[letterpaper]{article}
\usepackage{uai2019}
\usepackage[margin=1in]{geometry}

\usepackage{times}

\title{Joint Nonparametric Precision Matrix Estimation with Confounding}

\author{ {\bf Sinong Geng} \\
Department of Computer Science \\
Princeton University\\
Princeton, NJ 08544 \\
\And
{\bf Mladen Kolar}  \\
Booth School of Busines\\
The University of Chicago\\
Chicago, IL 60637 \\
\And
{\bf Oluwasanmi Koyejo}   \\
Department of Computer Science \\
University of Illinois Urbana-Champaign \\
Urbana, IL 61801\\
}

\usepackage{enumitem}

\usepackage{amsmath,amsthm,amssymb}
\allowdisplaybreaks
\usepackage{caption}
\usepackage{subcaption}

\usepackage{bm}

\usepackage{capt-of}

\usepackage{hyperref}
\usepackage[numbers]{natbib}
\usepackage{subcaption}

\newcommand{\norm}[1]{\lVert #1 \rVert}

\usepackage{mathtools}
\DeclarePairedDelimiter\abs{\lvert}{\rvert}%
\makeatletter
\let\oldabs\abs
\def\abs{\@ifstar{\oldabs}{\oldabs*}}

\makeatother

\newcommand{\be}{\mathbf{e}}
\newcommand{\bz}{\mathbf{z}}

\newcommand{\ba}{\mathbf{a}}

\newcommand{\bOmega}{\mathbf{\Omega}}
\newcommand{\bZ}{\mathbf{Z}}
\newcommand{\bM}{\mathbf{M}}
\newcommand{\bR}{\mathbf{R}}
\newcommand{\br}{\mathbf{r}}
\newcommand{\bS}{\mathbf{S}}
\newcommand{\bI}{\mathbf{I}}
\newcommand{\bbeta}{\bm{\beta}}

\newcommand{\bSigma}{\mathbf{\Sigma}}

\newcommand{\RR}{\mathbb{R}} 
\newcommand{\EE}{\mathbb{E}} 
\DeclareMathOperator*{\Var}{{\rm Var}}
\newcommand{\bmu}{\mathbf{\mu}}

\newcommand\given[1][]{\:#1\vert\:}
\DeclareMathOperator*{\argmin}{arg\,min}

\usepackage{physics}

\newtheorem{lemma}{Lemma}
\newtheorem{theorem}{Theorem}

\newtheorem{assumption}{Assumption}

\usepackage{algorithm}
\usepackage{algorithmicx}
\usepackage{algpseudocode}
\algrenewcommand\algorithmicindent{0.8em}

\usepackage{graphicx}

\usepackage{multirow}
\usepackage{float}

\usepackage{comment}
\usepackage[usenames]{color}

\begin{document}

\maketitle

\begin{abstract}
	We consider the problem of precision matrix estimation where, 
	due to extraneous confounding of the underlying precision matrix, 
	the data are independent but not identically distributed. While such confounding occurs in many scientific problems, 
	our approach is inspired by recent neuroscientific research suggesting that brain function, as measured using functional magnetic resonance imaging (fMRI), 
	is susceptible to confounding by physiological noise, such as breathing and subject motion. 
	Following the scientific motivation, we propose a graphical model, which in turn motivates 
	a joint nonparametric estimator.  We provide theoretical guarantees for the consistency 
	and the convergence rate of the proposed estimator. In addition, we demonstrate that the
	optimization of the proposed estimator can be transformed into a series of 
	linear programming problems and, thus, can be efficiently solved in parallel.
	Empirical results are presented using simulated and real brain imaging data and
	suggest that our approach improves precision matrix estimation as compared to baselines 
	when confounding is present.
\end{abstract}

\section{INTRODUCTION}
We consider the problem of precision matrix estimation where, 
due to extraneous confounding of the underlying precision matrix, 
the data are independent but not identically distributed.
While such confounding occurs in many scientific problems, 
our approach is inspired by applications to brain connectivity 
estimation from functional brain imaging. Functional brain connectivity 
has emerged as one of the most promising tools in the neuroscience 
toolbox for elucidating brain organization~\citep{biswal1995functional, fox2007spontaneous} 
and its relationship to behavior~\citep{sadaghiani2013functional,fornito2013,zalesky2012use,barch2013function}. 
Multiple studies suggest that functional brain connectivity may provide an accurate biomarker for cognitive 
disorders -- from Alzheimer's and schizophrenia, 
to autism and depression~\citep{price2014multiple,arbabshirani2011functional}, 
and may be the key to better understanding of these cognitive diseases. 

However, there is growing evidence that brain function as measured by
functional magnetic resonance imaging (fMRI), 
is susceptible to confounding by physiological noise, 
such as breathing and subject motion \citep{laumann2016stability, goto2016head}. 
In particular, these physiological signals cause complicated effects (usually non-linear),
and induce incorrect strong connectivity between brain areas. 
Perhaps most strikingly, some authors have suggested that much of what we now 
think of as functional connectivity might simply reflect these physiological
confounders.  A variety of techniques have been proposed in the neuroimaging literature for addressing 
the effects of physiological confounders~\citep{van2012influence, power2014methods,caballero2017methods}, 
most commonly by attempting to regress out their effects from the time series, or
using matrix factorization methods such as independent components analysis. 
While these methods may be effective for removing \textbf{\emph{linear}} confounding 
in the observed time series (or in the covariance), none of these address
our core concern of removing the effects of physiological confounding 
in the precision matrix, whose structure, in most of the cases,
directly corresponds to the connectivity of brain areas. 

In contrast to prior work, our manuscript addresses the physiological confounding
via a varying statistical graphical model. Specifically, by allowing underlying 
models to change over each observation, the proposed method directly addresses
the confounding of the precision matrices. We consider a novel method for precision 
matrix estimation from non-identically distributed data, where the extraneous factors
may \textbf{\emph{nonlinearly}} induce additive noise in the precision matrix. 
We propose a joint nonparametric estimator (JNE) that estimates the objective 
precision matrix using independent, but non-identically distributed (i.n.d.)~data. 
Surprisingly, the provided theoretical guarantees not only indicate the consistency of JNE,
but also point out a convergence rate comparable to that of 
state-of-the-art methods \textbf{\emph{without any}} confounding. 
An efficient optimization procedure based on linear programming that can easily 
be parallelized is applied to the parameter estimation. 
While the model is motivated and applied to neuroimaging, our results 
may be of more general interest to other applications where non-i.i.d.~signals 
induced by confounders are prevalent, such as financial and social 
network applications \citep{hong2012discovering,lee2015learning,geng2017efficient,geng2018stochastic,Geng2019Partially,Suggala2017Expxorcist}, 
and drug-condition interaction analysis~\citep{kuang2016baseline, kuang2016computational,kuang2017pharmacovigilance,Sun2015Learning}. 

We summarize the main contributions as follows:
\begin{itemize}
	\item We propose a graphical model for precision matrix estimation where the data are 
	independent, but not identically distributed, 
	due to systematic effects of confounders on the underlying precision matrix.
	\item We propose a joint nonparametric estimator and rigorously prove its 
	consistency and rate of convergence.
	\item We evaluate the resulting estimator using simulated and 
	real brain imaging data showing improved performance when confounding is present.
\end{itemize}

The paper is organized as follows. The overall approach for graphical
modeling with physiological confounders is outlined in Section~\ref{sec:model}.
Our proposed joint nonparametric estimator is outlined in Section~\ref{sec:estimation}. 
Three types of models related to the proposed one are discussed in Section~\ref{sec:related-models}. 
We investigate the consistency and convergence rate of the estimator
in Section~\ref{sec:consistency}. Experimental results on simulated 
and real brain imaging data are provided in Section~\ref{sec:experiments}.
We conclude in Section~\ref{sec:conclusion}.

\section{GRAPHICAL MODELING OF PHYSIOLOGICAL CONFOUNDERS}
\label{sec:model}

In this section, we introduce our model for brain connectivity analysis
with physiological confounders based on the 
framework of probabilistic graphical models. 
Undirected probabilistic graphical models are widely used to explore
and represent dependencies among random variables
\citep{Lauritzen1996Graphical}, in areas ranging
from image processing \citep{mignotte2000sonar} to 
multiple testing \citep{liu2014multiple,liu2016multiple} and 
computational biology \citep{friedman2004inferring, kuang2017screening}. 

An undirected
probabilistic graphical model consists of an undirected graph
${\cal G} = (V,E)$, where $V = \{1, \ldots, p\}$ is the vertex set and
$E \subset V \times V$ is the edge set, and a random vector
$\bZ = (Z_1, \ldots, Z_p) \in {\cal Z}^p \subseteq \RR^P$. Each coordinate
of the random vector $Z$ is associated with a vertex in $V$, and the
graph structure encodes the conditional independence assumptions
underlying the distribution of $Z$. In particular, 
$Z_j$ and $Z_{j'}$, with $j$, $j'$ $\in V$, are
conditionally independent given all the other variables if and only if
$(j,j') \not\in E$, that is, the nodes $j$ and $j'$ are not adjacent in
${\cal G}$. One of the fundamental problems in statistics is that of
learning the structure of ${\cal G}$ from i.i.d.~samples from $Z$ and
quantifying uncertainty of the estimated structure.
A recent review of algorithms for
learning the structure of graphical models is provided by~\cite{Drton2016Structure}.

Gaussian graphical models are commonly used for modeling continuous $\bZ$. In
this case, the edge set $E$ can be recovered by estimating the inverse
covariance matrix $\bOmega = \bSigma^{-1}$, known as the precision matrix. 
The sparsity pattern of the precision matrix encodes the edge set $E$, 
that is, $(j, j') \in E$ if and only if
$\bOmega_{jj'} \neq 0$. Therefore, sparse estimators of precision matrices,
like graphical Lasso \citep{friedman2008sparse} and CLIME \citep{Cai2011Constrained},
are commonly used for learning the structure of Gaussian graphical models.    

We assume that we are given $n$ independent observations $\left\{ \bz^i, g^i \right\}_{i \in [n]}$ from
the joint distribution of $(\bZ, G)$, where $\bZ \in \RR^p$ is a random
vector representing brain measurements and $G$ is a random variable representing confounders such as micro-motion. Rather than assuming that 
the confounders only linearly affect the mean of $\bZ$, 
as is commonly assumed in the literature~\citep{van2012influence, power2014methods,caballero2017methods}, 
we assume that it affects both the mean and 
the variance of $\bZ$. In particular, we assume that the conditional mean 
$\mu(g) = \EE(\bZ \mid G = g) \in \RR^p$  is a smooth function of the motion variable, and 
that the conditional covariance matrix 
\begin{equation*}
\bSigma(g) = \Var(\bZ \mid G=g)
\end{equation*}
has an inverse which
takes the form
\begin{equation}
\label{eq:model:precision}
\bOmega(g) = \bSigma^{-1}(g) = \bOmega^0 + \bR(g). 
\end{equation}
In the above model, $\bOmega^0$ is the target precision matrix while the term $\bR(g)$ is a nuisance component that arises due to physiological confounders. For the neuroscience application the sparsity pattern $\bOmega^0$ encodes the brain connectivity.

 Such an additive form for the precision matrix
 significantly generalizes existing models in the literature 
 \cite{van2012influence, power2014methods,caballero2017methods}, 
 where the following model is assumed:
\begin{equation}
\label{eq:lr-ggm}
\bZ = \bbeta^{\top}G + \bZ',
\end{equation} 
where $\bZ'$ follows a Gaussian graphical model with parameter $\bOmega$, 
and $\bbeta^{\top}G$ corresponds to linear confounding. 
Under the model in \eqref{eq:lr-ggm}, conditionally on $G=0$, 
$\bZ$ is equivalent to $\bZ'$ and
the target parameter for the non-confounded structure is $\bOmega$. 
However, for any $g$, the conditional distribution of $\bZ \given G=g$ is always 
a Gaussian graphical model with the precision matrix $\bOmega$.
In other words, such models indeed assume that the underlying precision matrices are \emph{not} affected by the confounders. 

It should be noted that recovering $\bOmega^0$ is impossible 
without any constrains on $\bR(\cdot)$.
Our identifiability condition for $\bOmega^0$ assumes that $\EE(\bR(G)) = \bm{0}$.
We justify this assumption from two perspectives. First, it is common in 
the nonparametric estimation literature to assume that the unknown curve has mean zero. 
Without this assumption, the constant term could be absorbed in the nonparametric component. 
Second, as we mentioned before, most of the widely used existing models
assume that the confounder does not affect the precision matrices of the 
observations, which is equivalent to assuming
\begin{equation}
\label{eq:zero-mean}
\bR(g) = 0,
\end{equation}
for any $g$, in \eqref{eq:model:precision}. 
As a result, the provided zero-expectation assumption 
relaxes \eqref{eq:zero-mean} asymptotically.  

Furthermore, we assume that the elements of $\bSigma(g)$ are smooth functions of $g$, 
which will facilitate our nonparametric estimation procedure. Practically, for the 
motivating micro-movement problem, the confounding is often assumed to be both 
linear and smooth \citep{van2012influence, power2014methods,caballero2017methods},
as in \eqref{eq:lr-ggm}. Therefore, we extend the linear assumption to a smooth, 
but nonparametric assumption. We formulate these assumption rigorously in Section~\ref{sec:assumptions}.

\section{RELATED MODELS}
\label{sec:related-models}

The model in \eqref{eq:model:precision} is motivated from the perspective of multi-task learning, where
for each task one has a parameter vector that can be decomposed into a common component, corresponding 
to $\bOmega^0$ in our setting, and a task specific component, corresponding to $\bR(g)$ in our setting 
\citep{Evgeniou2004Regularized}. The goal in multi-task learning is to improve prediction performance in 
supervised learning, while our goal is on identifying the common brain connectivity by removing the contamination
effects of motion. A big difference compared to the literature on multi-task learning is that here we have 
infinitely many tasks if $G$ has a density. While in multi-task learning one may not impose additional structure
on $\bR(g)$, here we assume smoothness over the motion variable $g$. The effect motion can be seen through $\bR(g)$, which
modulates the strength of edges in the true structure or adds spurious edges.

Our model \eqref{eq:model:precision} is closely related to the literature on time-varying undirected graphical models
\citep{Zhou08time, kolar2009sparsistent, kolar2011time, yin10nonparametric, kolar10nonparametric, 
Kolar2010Estimating,
kolar10estimating,
geng2018temporal,
kolar09nips_tv_paper,
song09time,
le09keller}. However, in this literature 
one is interested in estimating $\bOmega(g)$ as a function of time, 
without assuming existence of confounding effects. Simply averaging estimated $\bOmega(g)$ over $g$ 
will lead to inefficient estimators of $\bOmega^0$, as suggested in experiments in Section~\ref{sec:syn-data}. 

Another strand of the related literature focuses on estimation of multiple graphical 
models under the assumption that they
are  structurally similar \citep{Chiquet2011Inferring,Guo2011Joint,Danaher2011Joint,kolar13multiatticml,
Kolar2014Graph,Mohan2014Node,lee2015joint}.
This literature is similar to multi-task learning in that the goal is to leverage similarity between multiple 
related graphical models, with the focus on a finite, and usually very small, number of different graphs. 
This class of models turn out to be not applicable to our problem as outlined in Section~\ref{sec:syn-data}. 

\section{JOINT NONPARAMETRIC ESTIMATION TO THE PRECISION MATRIX}
\label{sec:estimation}

In this section, we propose an estimator for $\bOmega^0$ under the setting described in 
Section~\ref{sec:model}.
Since we only have one observation for any $G = g$ to estimate $\bOmega(g)$, we are going to pull 
the information from nearby observations. Specifically, we define a nonparametric estimator for 
the covariance matrix at $G = g$ as:
\begin{align}
	\label{eq:cov_estim}
	\begin{split}
		\bS(g) := \frac{\sum_{i=1}^n w_{i}(g)\bz^{i} \left(\bz^{i} \right)^\top}{ \sum_{i=1}^nw_{i}(g)} := \sum_{i=1}^n W_{i}(g)\bz^{i} \left(\bz^{i} \right)^\top,
	\end{split}
\end{align}
where 
\begin{equation*}
w_{i}(g) = \psi\left(\abs{g^i-g}/h\right)
\end{equation*}
with a symmetric density function $\psi(\cdot)$ and 
a user specified bandwidth $h > 0$. 
In practice, we select the bandwidth following the procedure in \cite{li2013optimal}. 
For convenience, we define $\bS^i := \bS(g^i)$, $W_{i,i'} = W_i(g^{i'})$, and 
$\bR^i = \bR(g^i)$.

With the covariance matrix estimator in \eqref{eq:cov_estim}, we define the proposed JNE as
\begin{align}
\label{eq:estimation}
\begin{split}
&\hat{\bOmega}^{0}, \left\{ \hat{\bR}^i \right\}_{i=1}^n   
\\=& \argmin_{\bM,\ \left \{\bR^{i} \right \}_{i=1}^n }\quad \left\{ \norm{\bM}_{L_1} +\frac{1}{n}\sum_{i=1}^n \norm{\bR^{i}}_{L_1} \right\}, 
\end{split}
\end{align}
with the following constraints:
\begin{align}
\label{eq:constraint}
    \begin{split}
&\quad \abs{ \bS^{i} \left(\bM+\bR^{i}\right)-\bI }_{\infty} \le \lambda, \quad  i=1,\ldots,n, \\
&\quad  \sum_{i=1}^n \bR^{i} = \bm{0},
\end{split}
\end{align}
where $\abs{\cdot}_\infty$ denotes the elementwise $L_\infty$ norm, 
$\abs{\cdot}_1$ the $L_1$ vector norm, 
and $\norm{\cdot}_{L_1}$ the $L_1$ matrix norm. 
The tuning parameter $\lambda$ is user-specified and controls how
close the estimated precision matrix is to the inverse of the kernel-estimated covariance matrix 
for each  $g \in \left\{g^i\right\}$. For ease of presentation, 
we use $\bM$ to denote the common part of precision matrices in 
the optimization program.

Although JNE has a similar form as CLIME~\citep{Cai2011Constrained}, it differs in two aspects: 
the sample covariance in CLIME is replaced by local kernel estimates; 
the constraint in \eqref{eq:constraint} incorporates all the local estimates. 
These modifications allow for pooling of the information from all samples, 
and as a result, JNE achieves a similar non-asymptotic sample complexity as CLIME 
for i.i.d. samples (see Section~\ref{sec:consistency-M}).

As in CLIME \citep{Cai2011Constrained}, we encourage the precision
matrix to be sparse, which in our case is equivalent to a sparse $\bM$ and sparse nuisance matrices $\bR^i$'s. 
We note that recovery of $\bOmega^0$ becomes increasingly more challenging as
nuisance matrices become more dense.
While the proposed estimator encourages sparsity, it is \emph{not} a necessary condition. In other words, 
the method is still applicable if the underlying $\bOmega^0$ is less sparse,  in which case the $\lambda$ that facilitates a consistent estimator should be chosen 
appropriately. Furthermore, 
according to Theorem~\ref{thm:consistency-M}, more data will be required for a 
less sparse $\bOmega^0$.

The objective function of JNE \eqref{eq:estimation} 
can be decomposed with respect to the columns of $\bM$, similar to the decomposition used in \cite{Cai2011Constrained}. 
In particular, for a matrix $\bM$, let $\bM_{*j}$ denote the $j$-th column vector. 
Then, for each $j=1,\ldots,p$, we consider the following $p$ minimization problems separately:
\begin{align}
\begin{split}
\label{eq:estimation-unconstrained-decomp}
& \argmin_{\bM_{*j},\ \left \{\bR^{i}_{*j} \right \}_{i=1}^n }\quad\left\{ \abs{\bM_{*j}}_{1} +\frac{1}{n}\sum_{i=1}^n \abs{\bR^{i}_{*j}}_{1} \right\} 
\end{split}
\end{align}
with the following constraints:
\begin{align*}
\begin{split}
&\quad \abs{ \bS^{i} \left(\bM_{*j}+\bR_{*j}^{i}\right)-\bI_{*j} }_{\infty} \le \lambda, \quad i=1,\ldots,n, 
\\
&\quad \sum_{i=1}^n \bR_{*j}^{i} = \bm{0}.
\end{split}
\end{align*}

The decomposed optimization tasks are instances of linear programs to which we apply 
the concurrent simplex method implemented 
in Gurobi \citep{gurobi}, which solves the problems on multiple threads simultaneously. 
Since the solution obtained by \eqref{eq:estimation-unconstrained-decomp} 
is not symmetric or positive definite in general, the final 
estimator is obtained after a symmetrization step as proposed by \cite{Cai2011Constrained}. 

\section{CONSISTENT ESTIMATION}
\label{sec:consistency}

In this section, we establish the consistency and the non-asymptotic sample complexity of JNE under mild assumptions. 

\subsection{ASSUMPTIONS}
\label{sec:assumptions}

We start off by listing assumptions, most of which are standard in kernel-based and 
CLIME-based methods. Assumption~\ref{ass:1} and~\ref{ass:2} specify the model and 
the structured effect of confounders on the covariance of $\bZ$ as described 
in Section~\ref{sec:model}. The assumptions enable identifiability of the underlying target.

\begin{assumption}\normalfont
	\label{ass:1}
	We assume that random variable $G$ has a density $f(g)$ with a compact support satisfying 
	\begin{align*}
		\inf f(g) \ge C_f > 0 \  \text{ and }\ \abs{f'_k(g) - f'_k(g')} \le C_d \abs{g - g'},
	\end{align*}
	for some constants $C_f$ and $C_d$. The conditional distribution of $\bZ$ given $G = g$
	is a Gaussian with the variance proxy $v$. That is,
	\begin{equation*}
		\Pr(|\ba^\top\bZ| > t) \leq c_1\exp{-c_2\cdot v\cdot t^2},
	\end{equation*}
	for any $\|\ba\|_2 = 1$ and some constants $c_1,c_2$. 
\end{assumption}

\begin{assumption}\normalfont
	\label{ass:2}
	We assume that 
	\begin{equation*}\Var(\bZ \mid G=g) = \bSigma(g)
	\end{equation*}
	with 
	\begin{equation*}
	    \bOmega(g) = \bSigma^{-1}(g) = \bOmega^0 + \bR(g)
	\end{equation*}
	and $\EE(\bOmega(G)) = \bOmega^0$. 
\end{assumption}

The following assumption ensures that the local covariance matrices are well behaved. A related assumption
was used in the analysis of CLIME \citep{Cai2011Constrained}.
\begin{assumption}\normalfont
	\label{ass:bound-mr}
	There exist $\Lambda_\infty, C_\infty \leq \infty$ such that 
	\begin{align*}
		\begin{split}
			\Lambda_\infty^{-1} \leq \inf_g \Lambda_{\min} (\bSigma{(g)})\leq \sup_g \Lambda_{\max} (\bSigma{(g)}) \leq \Lambda_\infty, 
		\end{split}
	\end{align*}
	and
	\begin{align*}
		\begin{split}
			\sup_g \abs{\bSigma{(g)}}_\infty \leq C_\infty,
		\end{split}
	\end{align*}
	where $\Lambda_{\min} (\cdot)$ and $\Lambda_{\max} (\cdot)$ denote the smallest and largest eigenvalues respectively. 
	Furthermore, there exist $C_{\bM}, C_{\bR} \leq \infty$ such that
	\[
	\sup_g \norm{\bR(g)}_{L_1} \le C_{\bR} \quad \text{ and } \quad \norm{\bOmega}_{L_1} \le C_{\bM}.
	\]
\end{assumption}
It should be noted that Assumption~\ref{ass:bound-mr} indeed indicates that 
$\bR(g)$'s and $\bOmega$ are sparse in terms of the $L_1$ norm. 
Also, $C_{\bR}$ and $C_{\bM}$ quantify the sparsity of the underlying precision matrices, 
and affect the convergence results in Theorem~\ref{thm:consistency-M}.
Back to the brain connectivity analysis, the sparsity assumption is reasonable,
as research suggests that the correlation among brain regions is a result of hub effects, 
and the conditional independence structure is likely to be sparse (or close to sparse)~\citep{hsieh2013big}.

Since we are using a local kernel estimator, 
we need the following three assumptions, which give regularity 
conditions that allow us to estimate $\bOmega^0$. 
Assumption~\ref{ass:kernel} imposes assumptions on the kernel function $\psi(\cdot)$ 
that are satisfied for a number of commonly used kernels.
\begin{assumption} \normalfont
	\label{ass:kernel}
	The kernel function $\psi(\cdot)$: $\mathbb{R} \rightarrow \mathbb{R}$ is a symmetric probability density function supported on $\left[ -1, 1\right]$. $h = O(\frac{1}{n^5})$. There exists a constant $C_\psi < \infty$, such that
	\begin{align*}
		\begin{split}
			\sup_x\abs{\psi(x)} \leq C_\psi\ \text{ and }\ \sup_x\psi(x)^2 \leq C_\psi .
		\end{split}
	\end{align*}
	Furthermore, $\psi(\cdot)$ is $C_L$-Lipschitz on  $\left[ -1, 1\right]$. That is 
	\begin{align*}
		\begin{split}
			\abs{\psi(x) - \psi(x')} \leq C_L \abs{x -x'}, \qquad x,x'\in \left[ -1, 1\right].
		\end{split}
	\end{align*}
\end{assumption}

The above assumption could be relaxed at the expense of more complicated proofs.

The next condition assumes smoothness of the conditional mean and variance of $\bZ$. 
\begin{assumption} \normalfont
	\label{ass:der}
	There exists a constant $C_{\bmu} < \infty$ such that
	\begin{align*}
		\begin{split}
			\sup_g
			\abs{\frac{d}{d g} \bmu{(g)}}_\infty \leq C_{\bmu}\  \text{ and }\  \sup_g\abs{\frac{d^2}{d g^2} \bmu{(g)}}_\infty \leq C_{\bmu}.
		\end{split}
	\end{align*}
	Furthermore, there exists a constant $C_{\bSigma} < \infty$ such that 
	\begin{align*}
		\begin{split}
			\sup_g
			\abs{\frac{d}{d g} \bSigma{(g)}}_\infty \leq C_{\bSigma}\  \text{ and }\  \sup_g\abs{\frac{d^2}{d g^2} \bSigma{(g)}}_\infty \leq C_{\bSigma}.
		\end{split}
	\end{align*}
\end{assumption}
Although complicated, the proposed assumptions are designed to relax the existing linear assumption \eqref{eq:lr-ggm} widely used in the neuroimaging literature~\citep{van2012influence, power2014methods,caballero2017methods}. Also, such assumptions are commonly used in nonparametric works to relax parametric assumptions. 

With the assumptions, we are ready to present our main result in the next section.

\subsection{CONVERGENCE RATE OF $\hat{\bOmega}^0$}\
\label{sec:consistency-M}

\begin{figure*}[t]
	\begin{minipage}[t]{0.45\linewidth}
		\centering
		\vspace{-50mm}
		\begin{tabular}{cccc}
			\hline \small
			\texttt{\# Nodes} & \texttt{Mm-CLIME} & \texttt{Ke-CLIME} & \texttt{Re-CLIME} \\
			\hline \normalsize
			10 & 0.078 & 0.078 & 0.039\\
			15 & 0.117 & 0.078 & 0.117\\
			20 & 0.117 & 0.117 & 0.117\\
			25 & 0.117 & 0.156 & 0.156\\
			\hline
		\end{tabular}
		\captionof{table}{Summary of the regularization parameters used for the considered methods using the synthetic datasets.}
		\label{tab:parameters}
	\end{minipage}
	\centering  
	\hspace{4mm}
	\begin{minipage}[b]{0.45\linewidth}
		\centering
		\includegraphics[scale=0.29]{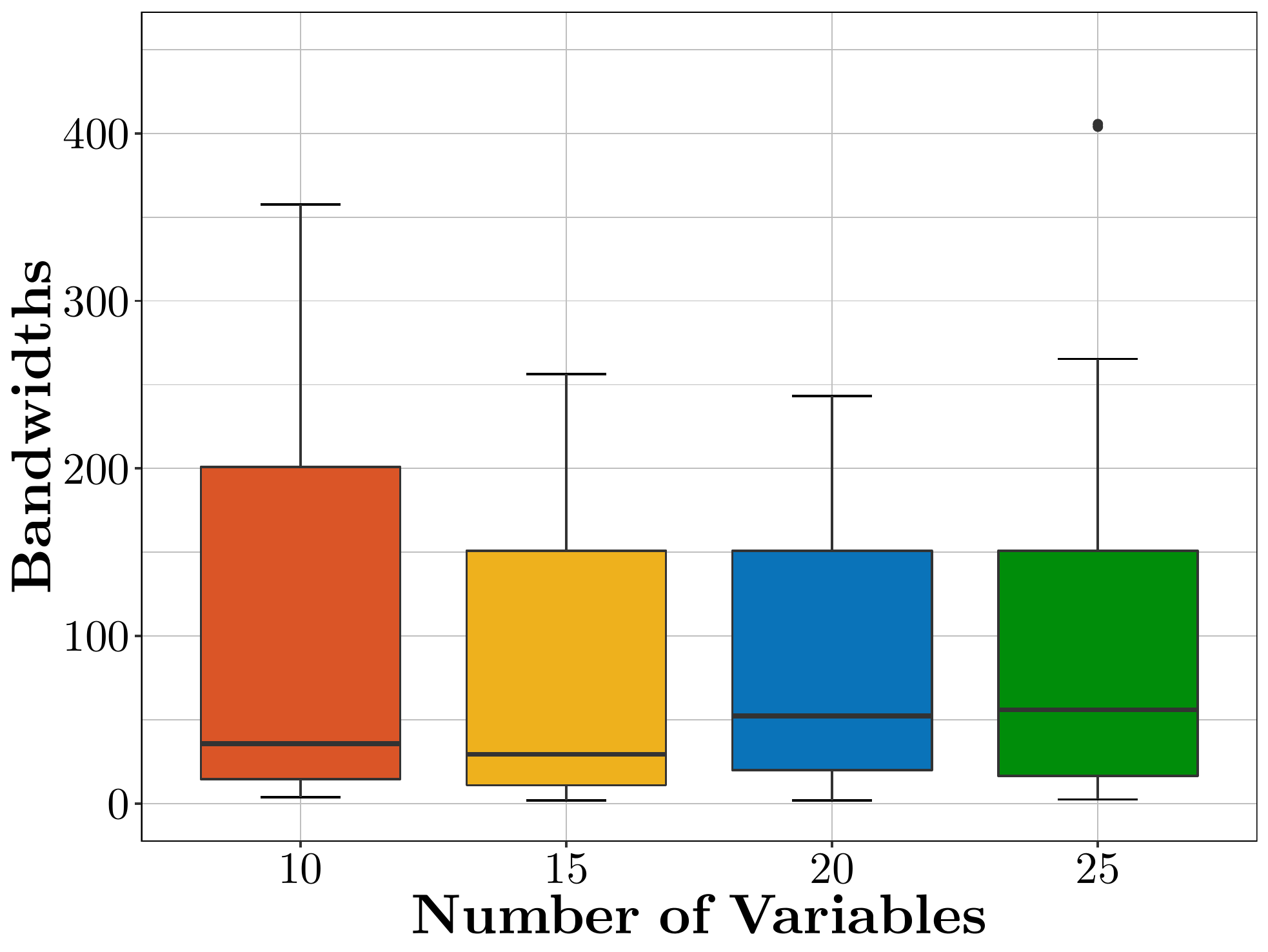}
		\caption{ Bandwidths for models with different numbers of variables for the synthetic datasets.}
		\label{fig:bw}
	\end{minipage}
	\vspace{2mm}
\end{figure*}

We provide the consistency and the non-asymptotic sample complexity of $\hat{\bOmega}^0$ in Theorem~\ref{thm:consistency-M}.

\begin{theorem}
	\label{thm:consistency-M}\normalfont
	Suppose that Assumptions~\ref{ass:1}--\ref{ass:der}
	are satisfied. 
	Given $n$ independent observations $\left\{ \bz^i, g^i \right\}_{i \in [n]}$, 
	we assume that there are $C_1>0$ and $C_2>0$, satisfying
	\begin{align*}
		\begin{split}
			\min_{\left\{ j,j'\given \exists i\in[n], \Omega_{jj'}(g^i) \neq 0 \right\}} & \sqrt{n^{-1} \sum_{i \in [n]} \left( \bOmega_{jj'}(g^i) \right)^2  } 
			\\&\qquad \qquad\quad\leq C_1 \sqrt{\log p} n^{-2/5},
		\end{split}
	\end{align*}
	where $n \geq C_2 d^{5/2}(\log p )^{5/4}$, 
	and 
	\begin{equation*}
		\lambda = \frac{\left(C_{\bM} + C_{\bR}\right)}{C_2} \sqrt{(r+1)\log p - \log C_1} n^{-2/5},
	\end{equation*}
	where $d$ denotes the maximum node degree of the graph. 
	
	Then, for any $r >0$, we have
	\begin{align}
		\label{eq:re-0-notin}
		\begin{split}
			&\abs{\hat{\bOmega}^0  - \bOmega^0 }_\infty 
			\\\leq & \frac{3\left(C_{\bM} + C_{\bR}\right)^2}{C_2} \sqrt{(r+1)\log p - \log C_1} n^{-2/5}
			\\&+\frac{C_{\bM}}{\sqrt{2}} \sqrt{(r+2)\log p + \log 2}n^{-1/2},
		\end{split}
	\end{align}
	with probability larger than $1- 2p^{-r}$.
\end{theorem}

In words, Theorem~\ref{thm:consistency-M} indicates that, with a high probability, the estimation error is bounded by 
$O\left(\sqrt{\frac{\log p}{n^{4/5}}}\right)$. Inevitably, it is slightly larger than 
$O\left(\sqrt{\frac{\log p}{n}}\right)$, i.e, the non-asymptotic sample complexity of CLIME with $n$ i.i.d. data.
Therefore, even in the presence of confounding, we can consistently and efficiently estimate
the underlying precision matrix $\bOmega^0$ that encodes the true connectivity pattern, 
with a convergence rate comparable to CLIME without any confounding.

\subsection{PROOF SKETCH}
We provide a rough sketch idea on the proof for Theorem \ref{thm:consistency-M}, and defer the details to the appendix.

To begin with, we need the following Lemma \ref{lem:s-1} on the convergence of $\bS(g)$ to $\bSigma(g)$,
since the derived estimator $\hat{\bOmega}^0$ is highly dependent on the nonparametric estimator $\bS(g)$.

\begin{lemma} \normalfont
	\label{lem:s-1}
	Suppose that the Assumption~\ref{ass:1} and Assumption \ref{ass:bound-mr} to \ref{ass:der} in Section~\ref{sec:assumptions} are satisfied, and that $C_1$ and $C_2$ are defined in Theorem~\ref{thm:consistency-M}. Then, for any $r>0$, and
	\begin{equation*}
	\delta = \frac{\sqrt{ (r+1) \log p - \log C_1}}{C_2}n^{-2/5},
	   \end{equation*}
	   the difference between $\bS(g)$ and $\bSigma(g)$ can be bounded in probability by
	\begin{align}
		\label{eq:s-1}
		\begin{split}
			\Pr{ \sup_{g} \abs{\bS(g) - \bSigma(g)}_\infty\ge \delta } \le p^{-r}.
		\end{split}
	\end{align}
\end{lemma}
\begin{proof}\small
	By following the rationale of the Lemma 8 in \cite{wang2014inference}, we can derive
	\begin{align*}
	\begin{split}
	\Pr& \left\{ \sup_{g} \abs{\bS(g) - \bSigma(g)}_\infty\ge \epsilon\right\} 
	\\&\le C_1 \exp(-C_2n^{4/5}\epsilon^2 + \log p).
	\end{split}
	\end{align*}
	Then, by letting $\epsilon = \sqrt{\frac{\log(p^{r+1} C_1^{-1})}{C_2 n^{4/5}}}$, we complete the proof.
\end{proof}

Lemma~\ref{lem:s-1} provides a uniform convergence result over $g$ of the local covariance estimator. 

Then, instead of directly studying $\abs{\hat{\bOmega}^0 - \bOmega^0}_\infty$, we start by bounding $\abs{\hat{\bOmega}^0 - \bM}_\infty$, where 
\begin{equation*}
\bM = \frac{\sum_{i \in [n]} \bOmega(g^i) }{n}.
\end{equation*}
Specifically, with Lemma \ref{lem:s-1}, we derive Lemma~\ref{lem:consistency-M} on the relationship between $\hat{\bOmega}^0$ and $\bM$.

\begin{figure*}[t]
	\centering
	\begin{minipage}[b]{0.28\linewidth}
		\centering
		\includegraphics[scale=0.25]{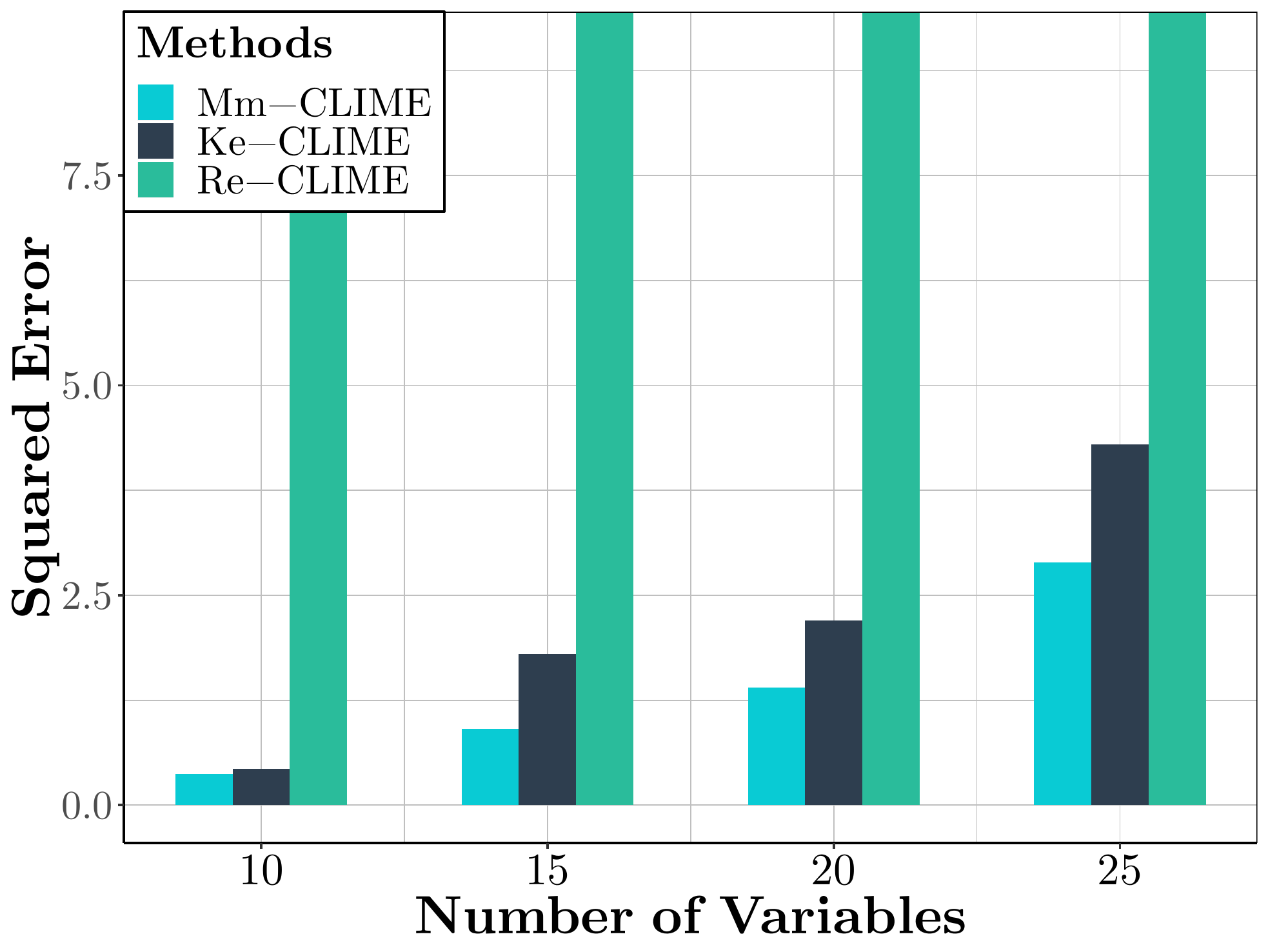}
		\caption{ Squared error for the considered methods with the selected $\lambda$'s using synthetic datasets. }
		\label{fig:auc-all-one}
	\end{minipage}
	\hspace{2mm}
	\centering 
	\vrule
	\hspace{2mm}
	\begin{minipage}[b]{0.6\linewidth}
		\begin{subfigure}{0.29\linewidth}
			\includegraphics[scale=0.23]{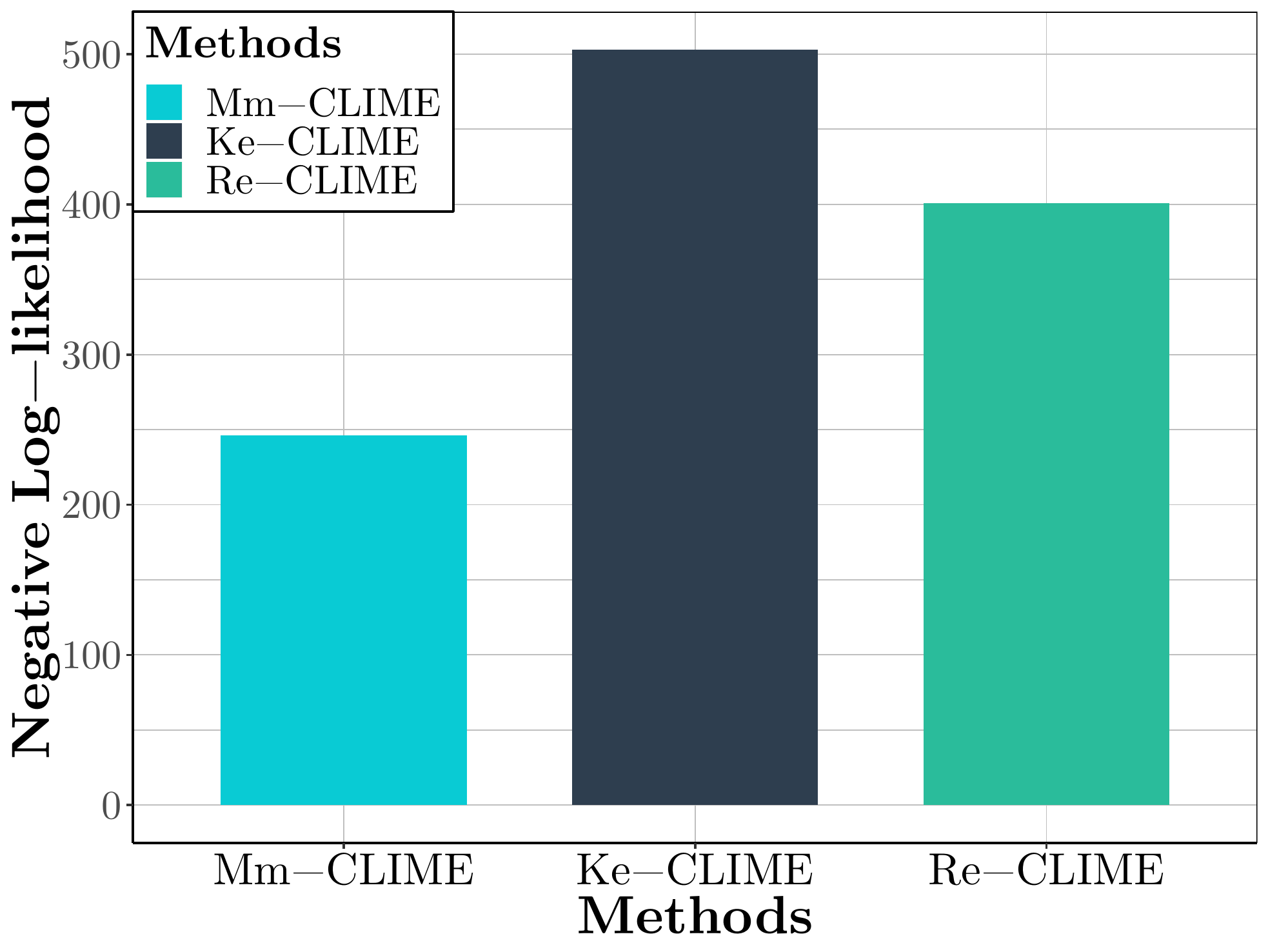}
		\end{subfigure}
		\hspace{20mm}
		\centering
		\begin{subfigure}{0.29\linewidth}
			\includegraphics[scale=0.23]{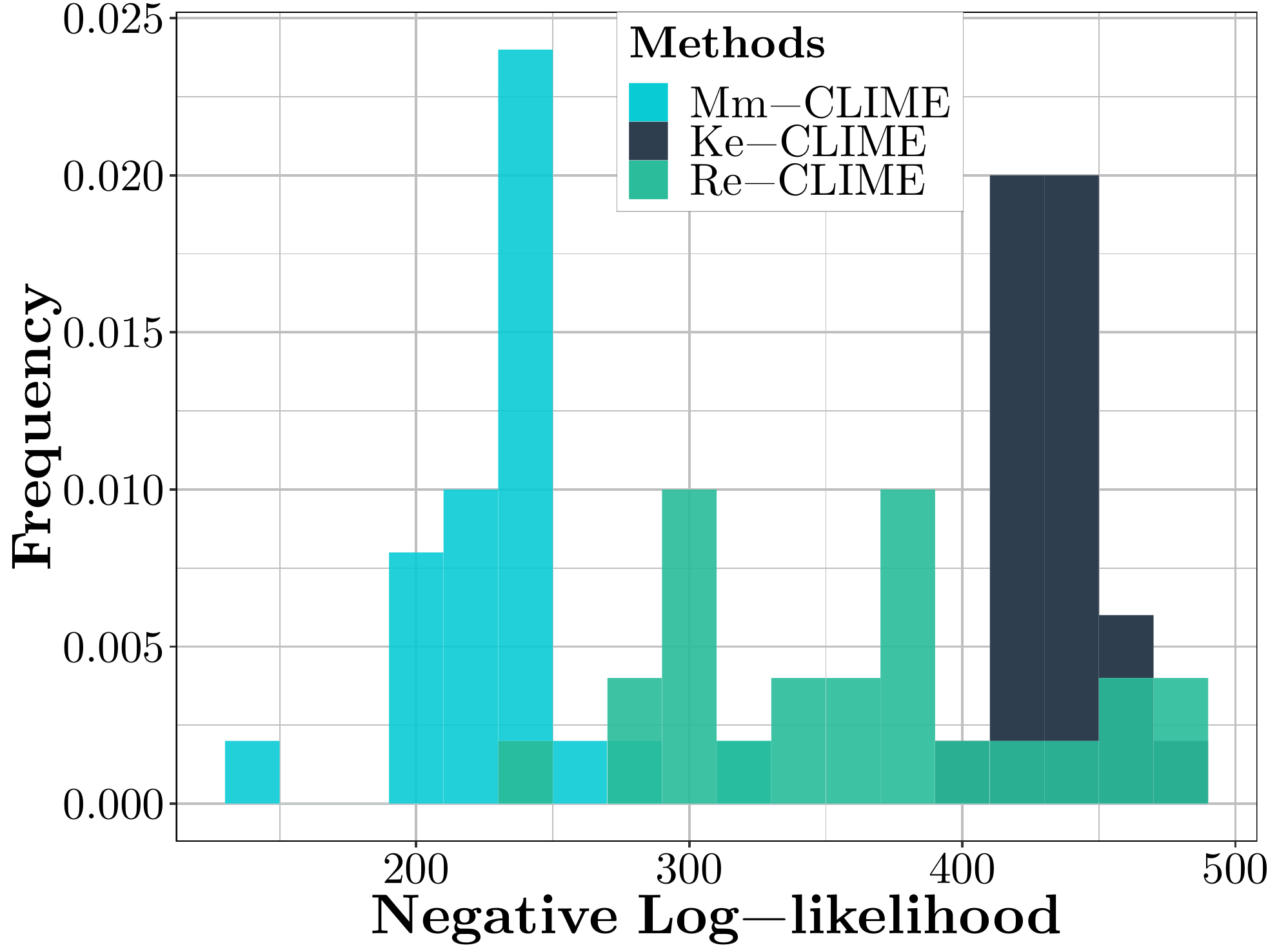}
		\end{subfigure}
		\hspace{5mm}
		\caption{Barplot of estimated negative log-likelihoods using the selected $\lambda$'s, histogram of the considered methods with all $\lambda$'s, using the COBRE dataset.}
		\label{fig:real-world-all}
	\end{minipage}
	\vspace{2mm}
\end{figure*}
\begin{figure}[t]
    \centering
    \includegraphics[scale=0.7]{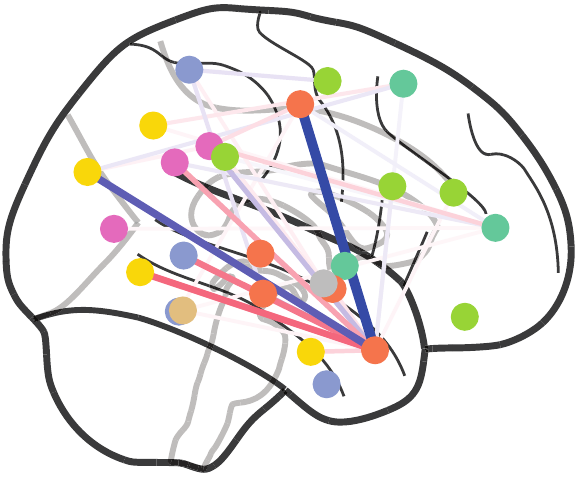}
    \caption{Glass brain figure of the estimated precision matrix using Mm-CLIME on the COBRE dataset (see manuscript for details).}
    \label{fig:glass-brain}
\end{figure}

\begin{figure*}[t]
		\centering
		\includegraphics[scale=0.23]{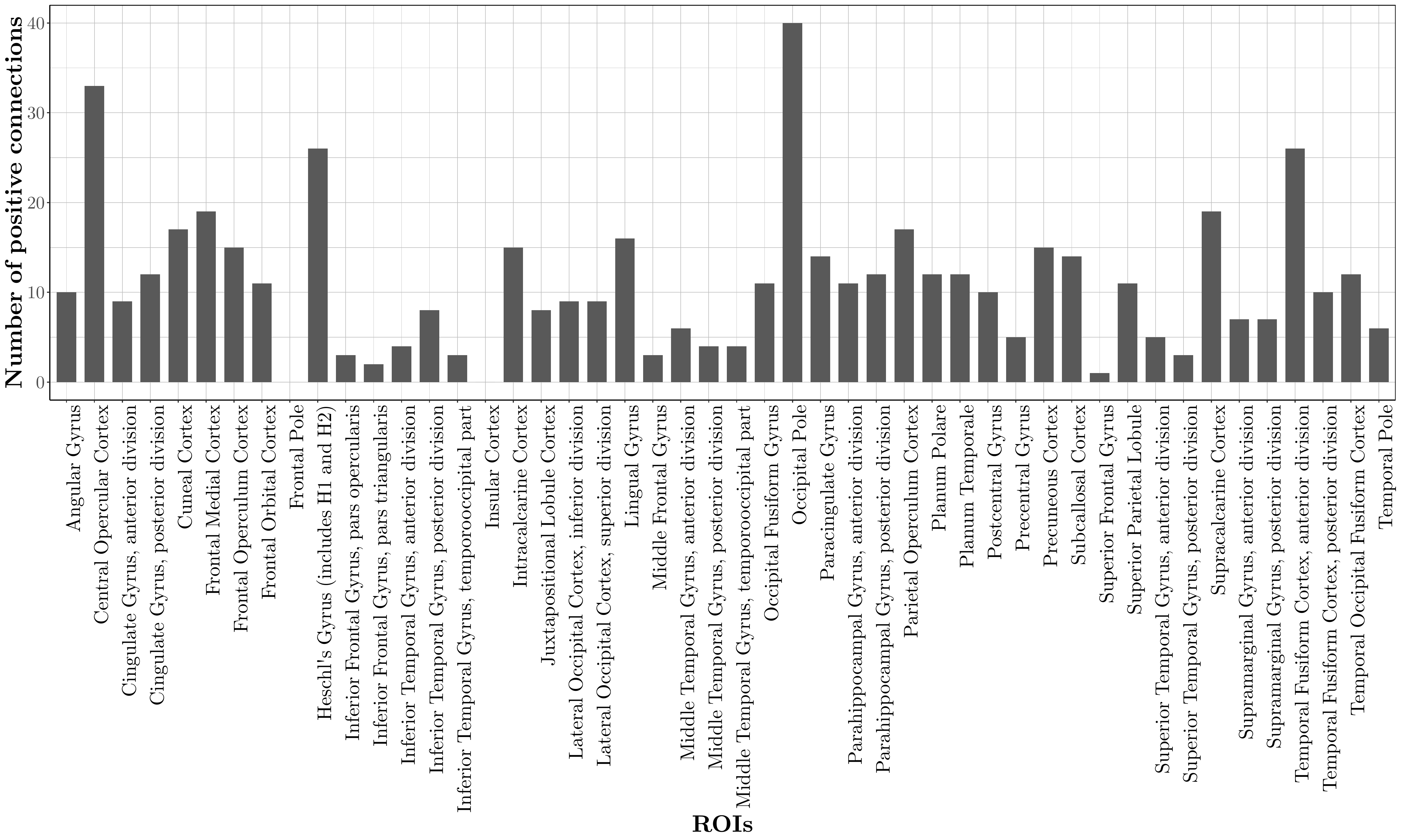}
		\caption{Number of positive connections across ROIs of the subjects diagnosed with schizophrenia recovered by the proposed Mm-CLIME method.}
		\label{fig:positive-connections-ex}
\end{figure*}~\begin{figure*}[t]
		\centering
		\includegraphics[scale=0.23]{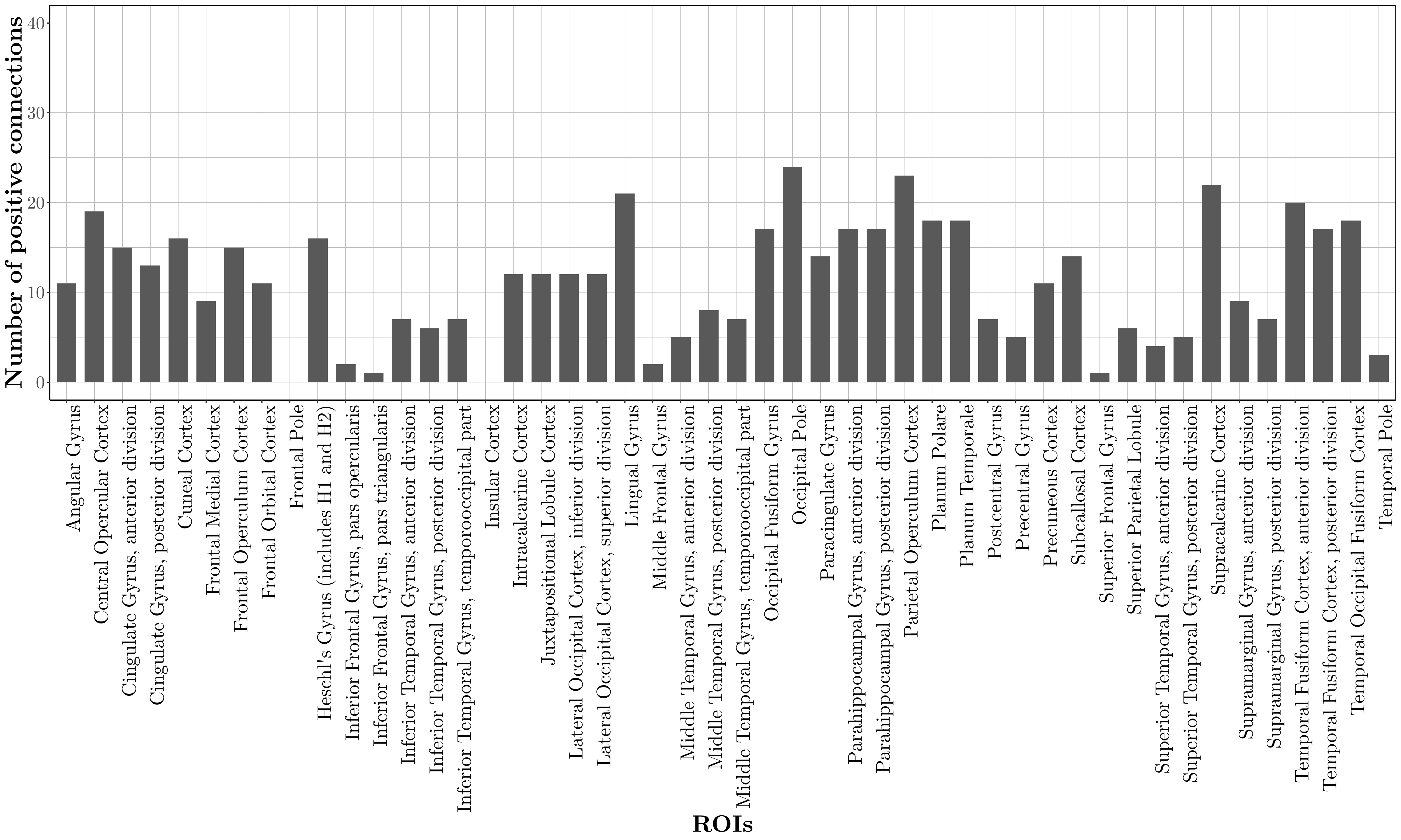}
		\caption{Number of positive connections across ROIs of the control subjects recovered by the proposed Mm-CLIME method.}
		\label{fig:positive-connections-co}
\end{figure*}

\begin{lemma}
	\label{lem:consistency-M}\normalfont
	Suppose that the conditions in Theorem \ref{thm:consistency-M} are satisfied. Then, for any $r >0$,
	\begin{align}
		\label{eq:re-0-notin-1}
		\begin{split}
			\abs{\hat{\bOmega}^0  - \bM}_\infty\qquad &
			\\ \qquad\le \frac{3\left(C_{\bM} + C_{\bR}\right)^2}{C_2}& \sqrt{ (r+1) \log p - \log C_1} n^{-2/5},
		\end{split}
	\end{align}
	with probability larger than $1- p^{-r}$.
\end{lemma}
The proof of Lemma~\ref{lem:consistency-M} is deferred to the appendix.

Finally, we bound $\abs{\bM - \bOmega^0}_\infty$.
	According to Lemma \ref{lem:consistency-M}, $\hat{\bOmega}^0$ converges to $\bM$. Therefore, we can prove the consistency of $\hat{\bOmega}^0$ by studying the relationship between $\bM$ and $\bOmega^0$. Specifically, by Hoeffding inequality, we have 
	\begin{equation}
	\label{eq:concentration-M}
	\Pr{\abs{\bOmega^0-\bM}_\infty \geq C_{\bM} \sqrt{\log(2p^{r+2})}(2n)^{-1/2}} \leq p^{-r},
	\end{equation}
	for any $r>0$.
	
	Combine \eqref{eq:concentration-M} with Lemma \ref{lem:consistency-M}, we have 
	\begin{align*}
	\begin{split}
	&\abs{\hat{\bOmega}^0  - \bOmega^0 }_\infty 
	\\\leq &\frac{3\left(C_{\bM} 
	+ C_{\bR}\right)^2}{C_2} \sqrt{\log(p^{r+1}) - \log(C_1)} n^{-2/5} \\&+\frac{C_{\bM}}{\sqrt{2}} \sqrt{\log(p^{r+2}) + \log(2)}n^{-1/2},
	\end{split}
	\end{align*}
	with probability $1-2p^{-r}$.

\section{EXPERIMENTS}
\label{sec:experiments}

In what follows, we will demonstrate that the proposed method efficiently recovers the target precision matrix when applied to synthetic data in Section~\ref{sec:syn-data}. Then, to illustrate that the proposed model is readily applicable to practical analysis of brain connectivity, we apply it to a resting-state functional magnetic resonance imaging dataset collected for the study of schizophrenia.    

\subsection{SYNTHETIC EXPERIMENTS}
\label{sec:syn-data}
In this section, we compare the proposed model with some existing models using synthetic data. Specifically, we consider a generative model where the underlying precision matrix varies smoothly with respect to the confounder variable $G$. Then, we generate samples with varying precision matrices via the following procedure:

	\begin{enumerate}
	\item We set the length of the multivariate random variable $\bZ$ to $p = 10, 15, 20, 25$, and implement the following steps separately.  
	\item We randomly generate $11$ precision matrices as anchors, 
	denoted by $\bOmega^{1}, \bOmega^{11}, \bOmega^{21},\bOmega^{31},\cdots, \bOmega^{101}$.
	Specifically, each element of the anchor precision matrices is drawn randomly to be non-zero
	with probability 0.2. The values of non-zero elements follow a uniform distribution. 
	\item Precision matrices between every two consecutive anchor precision matrices are constructed by linear interpolation.
	\item For each precision matrix, we generate $2$ independent zero-mean multivariate Gaussian samples,
	which constitute the synthetic dataset, $\mathbb{Z}$. 
	\item We set the target precision matrix as $\bOmega^0 = \frac{\sum_{i=1}^{101}\bOmega^{i} }{101}$,
	and the $\bOmega^0$ is thresholded for the sparsity. 
	\item The considered methods are applied to $\mathbb{Z}$ to estimate $\bOmega^0$.
\end{enumerate}

Note that by following the procedure above, the simulated model is equivalent to a generative model 
with fixed precision $\bOmega^0$ and additive confounding $\bR^i = \bOmega^0 - \bOmega^i$, 
where the superscripts indicate the confounder variable $G$.

The proposed method is based on a movement modeling method, 
and thus is denoted by Mm-CLIME. According to the analysis in Section~\ref{sec:model}, 
the competing methods fall into three categories: 
the precision matrix estimation with confounding (which is designed exactly for our problem setting), 
time-varying precision matrix estimation \citep{Zhou08time,kolar2011time,yin10nonparametric,kolar10nonparametric,Lu2015Posta,Kolar2010Estimating},
and multiple precision matrix estimation \citep{Chiquet2011Inferring,Guo2011Joint,Danaher2011Joint,Mohan2014Node,lee2015joint}. 
We also study two benchmarks: Re-CLIME and Ke-CLIME from the first two categories.

The baseline Re-CLIME refers to the procedure that linearly regresses 
out the confounding by $G$ from the observed samples first, and then
uses CLIME to estimate the precision matrix. As a precision matrix estimation method 
with confounding, Re-CLIME is widely applied to brain
functional connectivity analysis in practice \citep{van2012influence, power2014methods}.
Furthermore, we consider Ke-CLIME; a CLIME 
version of the method studied in \citep{kolar2011time} as a 
representative of time-varying precision matrix estimation. In this case, 
CLIME is separately applied to kernel estimators of the covariance matrix for each observation,
which results in $101$ estimated precision 
matrices, i.e.,  estimates of the time-varying precision matrices. 
Then, we use the average of the precision matrices as an estimator of $\bOmega^0$. 

To strike a fair comparison, we only select CLIME-based methods from each category of techniques. 
Otherwise, whether the efficiency gain of JNE comes from 
the joint estimation or CLIME will be unclear. Also, as suggested
by \cite{Cai2011Constrained}, CLIME has at least comparable performance to the graphical Lasso.

Multiple precision matrix estimation methods are not included due to numerical issues. 
We observe that such methods require sufficient samples for the estimation of each 
precision matrix. Empirical evaluation of one such approach~\cite{lee2015joint} most 
often resulted in ill-defined optimization problems, since we may have as few as two observations 
for each precision matrix.

The solutions obtained by the considered CLIME-based methods are not 
guaranteed to be symmetric or positive definite in general. 
To deal with this, a procedure provided in \cite{Cai2011Constrained} is 
used to symmetrize solution of each
method.

The Epanechnikov kernel is used with the bandwidth parameters selected 
by leave-one-out cross-validation. Since all the  considered 
kernel-based methods use the same estimator of the covariance matrix, 
the same bandwidth are used across all the methods. 
Different bandwidths are used to estimate each component of the covariance matrix, 
resulting in $p  (p-1)/2$ bandwidths.
Values of these bandwidths are provided in Figure~\ref{fig:bw}. 
The regularization parameters ($\lambda$'s) are 
determined by the Akaike information criterion (AIC). 
The selected $\lambda$'s are reported in Table~\ref{tab:parameters}.

The results are summarized in Figure \ref{fig:auc-all-one}. 
Notice that the proposed Mm-CLIME achieves the smallest squared error among 
all the considered methods. Furthermore, although Re-CLIME is widely used in 
practice to analyze the brain connectivity, it performs worst
in terms of accuracy in our experiments. Actually, this is not surprising, 
since the effects of confounders on the observed data are 
seldom linear, both in our synthetic setting and practical problems.  

Compared with Ke-CLIME, the proposed model adds an extra constraint in the optimization program and,
as a result, it can better take advantage of the assumption $E(R^i) = 0$ and  achieves more accurate estimation. 
The main intuition for this advantage is that with our procedure, we can choose a smaller bandwidth parameter $h$, 
which results in smaller bias and better performance compared to simple averaging. 
Note that when averaging the estimated $\Omega(g)$ over $g$ the bias due to kernel estimation 
does not get reduced. Furthermore, estimation of each individual $\Omega(g)$ requires a larger bandwidth, 
resulting in more bias.

\subsection{fMRI EXPERIMENTS}

Functional brain connectivity measures associations between brain areas,
and are thought to reflect communication and coordination between spatially distant 
neuronal populations -- corresponding to information processing pathways in the brain. 
There are a wide variety of techniques for estimating functional connectivity in the 
literature~\citep{sadaghiani2013functional,fornito2013,zalesky2012use,barch2013function}, 
most commonly via the correlation. However, there is increasing evidence that the precision 
matrix results in a more stable and effective connectivity estimate than the 
alternatives~\citep{varoquaux2010brain,poldrack2015long,shine2015estimation,shine2016temporal,andersen2018bayesian}. 
For this reason, we focus our attention on precision matrix estimation. 
We also focus on motion confounding -- one of the most pressing examples of physiological confounding in the literature.

Our experiments use data
from the Center for Biomedical Research Excellence 
(COBRE)\footnote{\url{https://github.com/SIMEXP/Projects/tree/master/cobre}}.
We compare the considered methods using preprocessed resting-state functional magnetic resonance images 
for $70$ patients diagnosed with schizophrenia, and $72$ healthy controls. 
Every patient has more than $150$ fMRI observations each with $22$ 
corresponding confounders. We use the $7$ confounders provided in the dataset
related to motion for the analysis. Furthermore, we apply Harvard-Oxford Atlas to generate 
48 atlas regions of interest (ROIs).

First we split the data from each patient into a training test with $70\%$ of the data and a test set with $30\%$ of the data. Then, we apply the method to training sets to estimate an unconfounded precision matrix for each patient. Further, we treat the observations of each test set as the samples of a multivariate normal distribution with the precision matrix as estimated, and calculate the log-likelihood to compare different methods.
 A glass brain
figure illustrating the achieved precision matrix of a patient with schizophrenia is provided in Figure~\ref{fig:glass-brain}.
Our hypothesis is that the confounding has a similar effect as additive noise. 
Thus, better estimation of the average precision will correspond to higher likelihood. 
The results using the $\lambda$ selected by AIC and the results using all considered $\lambda$'s 
are both summarized in Figure \ref{fig:real-world-all}. As expected, we find that the
proposed method achieves the highest log-likelihood on the held-out dataset, 
which suggests that it results in a better model of the brain connectivity than 
the baselines.

Next, we concatenated all the observations across subjects diagnosed with schizophrenia, and 
and all the healthy controls into two time series, then apply the proposed method separately to each concatenated time series. The results are summarized in
Figure~\ref{fig:positive-connections-ex} and Figure~\ref{fig:positive-connections-co},
where the number of positive connections for each ROI are reported. 
We notice that the connectivity of the ROIs of the subjects with schizophrenia are 
similar to that of the control group, except the Occipital Pole and Central Opercular Cortex 
have abnormally more positive connections with other ROIs, i.e., most highly connected areas. 
Interestingly, these two areas have been implicated in the literature as highly associated with 
schizophrenia~\citep{sheffield2015fronto,tohid2015alterations}. In summary, the proposed method models the fMRI accurately,
and detects the differences between the subjects with schizophrenia and the control group.



\section{Conclusion}
\label{sec:conclusion}
We developed a novel approach for precision matrix estimation where, 
due to extraneous confounding of the underlying precision matrix, 
the data are independent but not identically distributed. For this, 
we proposed a varying graphical model, 
and an associated joint nonparametric estimator. Our technical contributions included 
theoretical consistency and convergence rate guarantees for the proposed estimator, 
and an efficient optimization procedure. 
Empirical results were also presented using simulated and real brain imaging data, 
which suggests that our approach 
improves precision matrix estimation as compared to a variety of baselines. 
For future work, we plan to investigate
more complex hierarchical graphical models 
including more confounder effects, and joint estimation across groups to better estimate shared global structure.

We have focused on the estimation of the precision matrix 
under a specific model of confounding. An interesting future direction
is development of a goodness of fit test that would allow us to verify if 
the model specification is appropriate for data. 
Extending recent work on quantifying uncertainty in estimation
of edge parameters in undirected graphical models 
\cite{Wasserman2014Berry,Ren2013Asymptotic,Wang2016Inference,Barber2015ROCKET,
Yu2016Statistical,Kim2019Two,Yu2019Simultaneous,Jankova2018Inference}
to a setting with confounding is rather challenging and will be pursued elsewhere.

Empirical results were also presented suggesting an improvement for precision matrix estimation. 


\newpage
\subsubsection*{References}
\renewcommand\refname{\vskip -1cm}
{\small
\bibliography{MRI}
}

\newpage\clearpage\onecolumn
\input{supp.tex}

\end{document}

%% file: supp.tex
\newpage
	\appendix
	\addcontentsline{toc}{section}{Appendices}
	\section*{Supplements}

	\subsection*{Proof for Lemma~\ref{lem:consistency-M}}
	\label{sec:consistency-S}
	According to Lemma~\ref{lem:s-1}, for any $r >0$, 
	\begin{align}
	\begin{split}
	\label{eq:ineq-1}
	\sup_{g} \abs{\bS(g) - \bSigma(g)}_\infty\ge \sqrt{\log(p^{r+1}C_1^{-1})} C_2^{-1} n^{-2/5}
	\end{split}
	\end{align}
	with probability larger than $1- p^{-r}$. 
	
	For the ease of presentation, we now use $\bS^i$ to denote $\bS(g^i)$ and $\bR^i = \bOmega(g^i) - \bM$. Assuming that (\ref{eq:ineq-1}) holds, we can see that $\bM$ and $ \left \{\bR^{i}\right\}_{i=1}^n $ is a pair of feasible solution to (\ref{eq:estimation}) for
	\begin{align*}
	\begin{split}
	\abs{\bI - \bS^{i}\left( \bM + \bR^i \right)}_\infty & = \abs{\left( \bSigma(g^i) - \bS^{i} \right) \left( \bM + \bR^i\right)}_\infty \\
	& \le (C_{\bM} + C_{\bR}) \sqrt{\log(p^{r+1}C_1^{-1})} C_2^{-1} n^{-2/5} 
	\\&= \lambda.
	\end{split}
	\end{align*}
	
	Then, we consider $\abs{\sum_{i=1} ^n \left( \bM + \bR^{i}  - \hat{\bOmega}^0 - \hat{\bR}^{i}\right) \be_j }_\infty$, which satisfies :
	
	\begin{align*}
	\begin{split}
	\abs{\sum_{i=1} ^n \left( \bM + \bR^{i}  - \hat{\bM} - \hat{\bR}^{i}\right) \be_j }_\infty \le &\abs{\sum_{i=1} ^n  \left( \bM + \bR^i\right ) \left( \bSigma^{i} - \bS^{i} \right)\left(\hat{\bOmega}^0 + \hat{\bR}^{i}\right) \be_j }_\infty \\
	&+ \abs{\sum_{i=1} ^n   \left( \bM + \bR^{i}\right ) \left(\bS^{i}\left(\hat{\bOmega}^0 + \hat{\bR}^{i}\right) -\bI \right ) \be_j }_\infty,
	\end{split}
	\end{align*}
	where $\bI$ is the identity matrix. 
	
	For $\abs{\sum_{i=1} ^n  \left( \bM + \bR^i\right ) \left( \bSigma(g^i) - \bS^{i} \right)\left(\hat{\bOmega}^0 + \hat{\bR}^{i}\right) \be_j }_\infty$, according to the Assumption \ref{ass:bound-mr}, we have:
	\begin{align*}
	\begin{split}
	&\abs{\sum_{i=1} ^n  \left( \bM + \bR^i\right ) \left( \bSigma(g^i) - \bS^{i} \right)\left(\hat{\bOmega}^0 + \hat{\bR}^{i}\right) \be_j }_\infty\\
	\le & \norm{\bM} \left \{\sum_{i=1} ^n \abs{ \bSigma(g^i) - \bS^{i}}_\infty  \abs{\hat{\bm{\omega}}_j}_1  + \sum_{i=1} ^n\abs{ \bSigma(g^i) - \bS^{k} }_\infty  \abs{\hat{\br}^{i}_j}_1 \right \}\\
	&+ \sum_{i=1} ^n\norm{\bR^{i}}_{L_1}\abs{ \left(\bSigma(g^i) - \bS^{i} \right)}_\infty  \abs{\hat{\bm{\omega}}_j}_1  +  \sum_{i=1} ^n\norm{\bR^{i}}_{L_1}\abs{ \left(\bSigma^{i} - \bS^{i} \right)}_\infty  \abs{\hat{\br}^{i}_j}_1\\
	\le &C_{\bM}   \left \{ n\max_i\abs{\left(\bSigma(g^i) - \bS^{i} \right)}_\infty  \abs{\hat{\bm{\omega}}_j}_1  + \max_i\abs{ \left(\bSigma(g^i) - \bS^{k} \right)}_\infty \sum_{i=1} ^n \abs{\hat{\br}^{i}_j}_1 \right \}\\
	&+ \abs{\hat{\bm{\omega}}_j}_1 \max_i\abs{\left(\bSigma(g^i) - \bS^{i} \right)}_\infty\sum_{i=1} ^n  \norm{\bR^{i}}_{L_1} + \max_i\abs{\left(\bSigma(g^i) - \bS^{i} \right)}_\infty \sum_{i=1} ^n\norm{\bR^{i}}_{L_1} \abs{\hat{\br}^{i}_j}_1 \\
	\le & C_{\bM}\sqrt{\log(p^{r+1}C_1^{-1})} C_2^{-1} n^{3/5}\abs{\hat{\bm{\omega}}_j}_1 + C_{\bM} \sqrt{\log(p^{r+1}C_1^{-1})} C_2^{-1} n^{-2/5}\sum_{i=1}^n  \abs{\hat{\br}_j^{i}}_1  \\
	& + C_\bR\sqrt{\log(p^{r+1}C_1^{-1})} C_2^{-1} n^{3/5}\abs{\hat{\bm{\omega}}_j}_1 + C_\bR \sqrt{\log(p^{r+1}C_1^{-1})} C_2^{-1} n^{-2/5}\sum_{i=1}^n  \abs{\hat{\br}_j^{i}}_1 
	\\ = & 2\left( C_{\bM} + C_{\bR} \right) \sqrt{\log(p^{r+1}C_1^{-1})} C_2^{-1} n^{3/5} \left(   \abs{\hat{\bm{\omega}}_j}_1 +\frac{\sum_{i=1}^n\abs{\hat{\br}_j^{i}}_1}{n}  \right),
	\end{split}
	\end{align*}
	where $\hat{\bm{\omega}}_j$ denotes the $j$th column of $\hat{\bOmega}^0$
	
	Since $\bM$ and $ \left \{\bR^{i}\right\}_{i=1}^n $ is a set of feasible solution to (\ref{eq:estimation}), combined with the definition of  \eqref{eq:estimation}, we can derive the upperbound of $\abs{\sum_{i=1} ^n  \left( \bM + \bR^i\right ) \left( \bSigma(g^i) - \bS^{i} \right)\left(\hat{\bOmega}^0 + \hat{\bR}^{i}\right) \be_j }_\infty$ as:
	\begin{align}
	\label{eq:part1}
	\begin{split}
	& \abs{\sum_{i=1} ^n  \left( \bM + \bR^i\right ) \left( \bSigma(g^i) - \bS^{i} \right)\left(\hat{\bOmega}^0 + \hat{\bR}^{i}\right) \be_j }_\infty
	\\ \le & 2\left( C_{\bM} + C_{\bR} \right) \sqrt{\log(p^{r+1}C_1^{-1})} C_2^{-1} n^{3/5} \left(   \abs{\hat{\bm{\omega}}_j}_1+\frac{\sum_{i=1}^n\abs{\hat{\br}_j^{i}}_1}{n}  \right)
	\\ \leq &2\left( C_{\bM} + C_{\bR} \right)^2 \sqrt{\log(p^{r+1}C_1^{-1})} C_2^{-1} n^{3/5},
	\end{split}
	\end{align}
	where the second inequality is due to the Assumption \ref{ass:bound-mr}.
	
	Furthermore, for $ \abs{\sum_{i=1} ^n   \left( \bM + \bR^{i}\right ) \left(\bS^{i}\left(\hat{\bOmega}^0 + \hat{\bR}^{i}\right) -\bI \right ) \be_j }_\infty$, we have:
	
	\begin{align}
	\label{eq:part2}
	\begin{split}
	&\abs{\sum_{i=1} ^n   \left( \bM + \bR^{i}\right ) \left(\bS^{i}\left(\hat{\bOmega}^0 + \hat{\bR}^{i}\right) -\bI \right ) \be_j }_\infty\\
	\le& \norm{\bM}_{L_1} \abs{\sum_{i=1} ^n  \left(\bS^{i}\left(\hat{\bOmega}^0 + \hat{\bR}^{i}\right) -\bI \right ) \be_j }_\infty
	+\sum_{i=1} ^n \norm{\bR^i}_{L_1} \abs{  \left(\bS^{i}\left(\hat{\bOmega}^0 + \hat{\bR}^{i}\right) -\bI \right ) \be_j }_\infty
	\\ \leq & \lambda \left(  \norm{\bM}_{L_1}+ \sum_{i=1} ^n \norm{\bR^i}_{L_1} \right)
	\\ = & ( n C_{\bR} +  C_\bM) (C_{\bM} + C_{\bR}) \sqrt{\log(p^{r+1}C_1^{-1})} C_2^{-1} n^{-2/5}. 
	\end{split}
	\end{align}
	
	Therefore, combining (\ref{eq:part1}) and (\ref{eq:part2}), we can prove the consistency of $\hat{\bOmega}^0$ by
	\begin{align}
	\label{eq:sum}
	\begin{split}
	&\frac{\abs{\sum_{i=1} ^n \left( \bM + \bR^{i}  - \hat{\bOmega}^0 - \hat{\bR}^{i}\right) \be_j }_\infty}{n} 
	\\= & \abs{\hat{\bOmega}^0 -\bM }_{\infty}
	\\ \leq &\left(C_{\bM} + C_{\bR}\right) \left( 3C_{\bR} + (2+n^{-1})C_{\bM} \right) \sqrt{\log(p^{r+1}C_1^{-1})} C_2^{-1} n^{-2/5},
	\end{split}
	\end{align}
	with probability larger than $1- p^{-r}$.

%% file: MRI.bbl
\begin{thebibliography}{71}
\providecommand{\natexlab}[1]{#1}
\providecommand{\url}[1]{\texttt{#1}}
\expandafter\ifx\csname urlstyle\endcsname\relax
  \providecommand{\doi}[1]{doi: #1}\else
  \providecommand{\doi}{doi: \begingroup \urlstyle{rm}\Url}\fi

\bibitem[Andersen et~al.(2018)Andersen, Winther, Hansen, Poldrack, and
  Koyejo]{andersen2018bayesian}
M.~Andersen, O.~Winther, L.~K. Hansen, R.~Poldrack, and O.~Koyejo.
\newblock Bayesian structure learning for dynamic brain connectivity.
\newblock In \emph{AISTATS}, 2018.

\bibitem[Arbabshirani and Calhoun(2011)]{arbabshirani2011functional}
M.~R. Arbabshirani and V.~D. Calhoun.
\newblock Functional network connectivity during rest and task: comparison of
  healthy controls and schizophrenic patients.
\newblock In \emph{EMBC}, 2011.

\bibitem[Barber and Kolar(2018)]{Barber2015ROCKET}
R.~F. Barber and M.~Kolar.
\newblock Rocket: Robust confidence intervals via kendall's tau for
  transelliptical graphical models.
\newblock \emph{Ann. Statist.}, 46\penalty0 (6B):\penalty0 3422--3450, 2018.

\bibitem[Barch et~al.(2013)Barch, Burgess, Harms, Petersen, Schlaggar,
  Corbetta, Glasser, Curtiss, Dixit, Feldt, et~al.]{barch2013function}
D.~M. Barch, G.~C. Burgess, M.~P. Harms, S.~E. Petersen, B.~L. Schlaggar,
  M.~Corbetta, M.~F. Glasser, S.~Curtiss, S.~Dixit, C.~Feldt, et~al.
\newblock Function in the human connectome: task-fmri and individual
  differences in behavior.
\newblock \emph{Neuroimage}, 80:\penalty0 169--189, 2013.

\bibitem[Biswal et~al.(1995)Biswal, Zerrin~Yetkin, Haughton, and
  Hyde]{biswal1995functional}
B.~Biswal, F.~Zerrin~Yetkin, V.~M. Haughton, and J.~S. Hyde.
\newblock Functional connectivity in the motor cortex of resting human brain
  using echo-planar mri.
\newblock \emph{Magnetic resonance in medicine}, 34\penalty0 (4):\penalty0
  537--541, 1995.

\bibitem[Caballero-Gaudes and Reynolds(2017)]{caballero2017methods}
C.~Caballero-Gaudes and R.~C. Reynolds.
\newblock Methods for cleaning the bold fmri signal.
\newblock \emph{Neuroimage}, 154:\penalty0 128--149, 2017.

\bibitem[Cai et~al.(2011)Cai, Liu, and Luo]{Cai2011Constrained}
T.~T. Cai, W.~Liu, and X.~Luo.
\newblock A constrained $\ell_1$ minimization approach to sparse precision
  matrix estimation.
\newblock \emph{J. Am. Stat. Assoc.}, 106\penalty0 (494):\penalty0 594--607,
  2011.

\bibitem[Chiquet et~al.(2011)Chiquet, Grandvalet, and
  Ambroise]{Chiquet2011Inferring}
J.~Chiquet, Y.~Grandvalet, and C.~Ambroise.
\newblock Inferring multiple graphical structures.
\newblock \emph{Stat. Comput.}, 21\penalty0 (4):\penalty0 537--553, 2011.

\bibitem[{Danaher} et~al.(2014){Danaher}, {Wang}, and Witten]{Danaher2011Joint}
P.~{Danaher}, P.~{Wang}, and D.~M. Witten.
\newblock The joint graphical lasso for inverse covariance estimation across
  multiple classes.
\newblock \emph{J. R. Stat. Soc. B}, 76\penalty0 (2):\penalty0 373--397, 2014.

\bibitem[Drton and Maathuis(2017)]{Drton2016Structure}
M.~Drton and M.~H. Maathuis.
\newblock Structure learning in graphical modeling.
\newblock \emph{Annual Review of Statistics and Its Application}, 4\penalty0
  (1):\penalty0 365--393, 2017.

\bibitem[Evgeniou and Pontil(2004)]{Evgeniou2004Regularized}
T.~Evgeniou and M.~Pontil.
\newblock Regularized multi-task learning.
\newblock In \emph{SIGKDD}, 2004.

\bibitem[Fornito et~al.(2013)Fornito, Lee, and Breakspear]{fornito2013}
A.~Fornito, A.~Z. Lee, and M.~Breakspear.
\newblock Graph analysis of the human connectome: Promise, progress, and
  pitfalls.
\newblock \emph{NeuroImage}, 80:\penalty0 426--444, 2013.

\bibitem[Fox and Raichle(2007)]{fox2007spontaneous}
M.~D. Fox and M.~E. Raichle.
\newblock Spontaneous fluctuations in brain activity observed with functional
  magnetic resonance imaging.
\newblock \emph{Nature Reviews Neuroscience}, 8\penalty0 (9):\penalty0
  700--711, 2007.

\bibitem[Friedman et~al.(2008)Friedman, Hastie, and
  Tibshirani]{friedman2008sparse}
J.~H. Friedman, T.~J. Hastie, and R.~J. Tibshirani.
\newblock Sparse inverse covariance estimation with the graphical lasso.
\newblock \emph{Biostatistics}, 9\penalty0 (3):\penalty0 432--441, 2008.

\bibitem[Friedman(2004)]{friedman2004inferring}
N.~Friedman.
\newblock Inferring cellular networks using probabilistic graphical models.
\newblock \emph{Science}, 303\penalty0 (5659):\penalty0 799--805, 2004.

\bibitem[Geng et~al.(2017)Geng, Kuang, and Page]{geng2017efficient}
S.~Geng, Z.~Kuang, and D.~Page.
\newblock An efficient pseudo-likelihood method for sparse binary pairwise
  markov network estimation.
\newblock \emph{arXiv:1702.08320}, 2017.

\bibitem[Geng et~al.(2018{\natexlab{a}})Geng, Kuang, Liu, Wright, and
  Page]{geng2018stochastic}
S.~Geng, Z.~Kuang, J.~Liu, S.~Wright, and D.~Page.
\newblock Stochastic learning for sparse discrete markov random fields with
  controlled gradient approximation error.
\newblock In \emph{UAI}, 2018{\natexlab{a}}.

\bibitem[Geng et~al.(2018{\natexlab{b}})Geng, Kuang, Peissig, and
  Page]{geng2018temporal}
S.~Geng, Z.~Kuang, P.~Peissig, and D.~Page.
\newblock Temporal poisson square root graphical models.
\newblock In \emph{ICML}, 2018{\natexlab{b}}.

\bibitem[Geng et~al.(2019)Geng, Yan, Kolar, and Koyejo]{Geng2019Partially}
S.~Geng, M.~Yan, M.~Kolar, and S.~Koyejo.
\newblock Partially linear additive {G}aussian graphical models.
\newblock In \emph{ICML 36}, 2019.

\bibitem[Goto et~al.(2016)Goto, Abe, Miyati, Yamasue, Gomi, and
  Takeda]{goto2016head}
M.~Goto, O.~Abe, T.~Miyati, H.~Yamasue, T.~Gomi, and T.~Takeda.
\newblock Head motion and correction methods in resting-state functional mri.
\newblock \emph{Magnetic Resonance in Medical Sciences}, 15\penalty0
  (2):\penalty0 178--186, 2016.

\bibitem[Guo et~al.(2011)Guo, Levina, Michailidis, and Zhu]{Guo2011Joint}
J.~Guo, E.~Levina, G.~Michailidis, and J.~Zhu.
\newblock Joint estimation of multiple graphical models.
\newblock \emph{Biometrika}, 98\penalty0 (1):\penalty0 1--15, 2011.

\bibitem[Gurobi~Optimization(2016)]{gurobi}
I.~Gurobi~Optimization.
\newblock Gurobi optimizer reference manual, 2016.

\bibitem[Hong et~al.(2012)Hong, Ahmed, Gurumurthy, Smola, and
  Tsioutsiouliklis]{hong2012discovering}
L.~Hong, A.~Ahmed, S.~Gurumurthy, A.~J. Smola, and K.~Tsioutsiouliklis.
\newblock Discovering geographical topics in the twitter stream.
\newblock In \emph{WWW 21}, 2012.

\bibitem[Hsieh et~al.(2013)Hsieh, Sustik, Dhillon, Ravikumar, and
  Poldrack]{hsieh2013big}
C.-J. Hsieh, M.~A. Sustik, I.~S. Dhillon, P.~K. Ravikumar, and R.~Poldrack.
\newblock Big \& quic: Sparse inverse covariance estimation for a million
  variables.
\newblock In \emph{NIPS}, 2013.

\bibitem[Jankov\'{a} and van~de Geer(2019)]{Jankova2018Inference}
J.~Jankov\'{a} and S.~van~de Geer.
\newblock Inference in high-dimensional graphical models.
\newblock In \emph{Handbook of graphical models}, Chapman \& Hall/CRC Handb.
  Mod. Stat. Methods, pages 325--349. CRC Press, Boca Raton, FL, 2019.

\bibitem[Kim et~al.(2019)Kim, Liu, and Kolar]{Kim2019Two}
B.~Kim, S.~Liu, and M.~Kolar.
\newblock Two-sample inference for high-dimensional markov networks.
\newblock \emph{arXiv 1905.00466}, 2019.

\bibitem[Kolar and Xing(2011)]{kolar2011time}
M.~Kolar and E.~P. Xing.
\newblock On time varying undirected graphs.
\newblock In \emph{AISTATS}, 2011.

\bibitem[Kolar and Xing(2012)]{kolar10estimating}
M.~Kolar and E.~P. Xing.
\newblock Estimating networks with jumps.
\newblock \emph{Electron. J. Stat.}, 6:\penalty0 2069--2106, 2012.

\bibitem[Kolar et~al.(2009{\natexlab{a}})Kolar, Song, and
  Xing]{kolar09nips_tv_paper}
M.~Kolar, L.~Song, and E.~P. Xing.
\newblock Sparsistent learning of varying-coefficient models with structural
  changes.
\newblock In Y.~Bengio, D.~Schuurmans, J.~D. Lafferty, C.~K.~I. Williams, and
  A.~Culotta, editors, \emph{NIPS}, pages 1006--1014, 2009{\natexlab{a}}.

\bibitem[Kolar et~al.(2009{\natexlab{b}})Kolar, Song, and
  Xing]{kolar2009sparsistent}
M.~Kolar, L.~Song, and E.~P. Xing.
\newblock Sparsistent learning of varying-coefficient models with structural
  changes.
\newblock In \emph{NIPS}, 2009{\natexlab{b}}.

\bibitem[Kolar et~al.(2010{\natexlab{a}})Kolar, Parikh, and
  Xing]{kolar10nonparametric}
M.~Kolar, A.~P. Parikh, and E.~P. Xing.
\newblock On sparse nonparametric conditional covariance selection.
\newblock In J.~F{\"u}rnkranz and T.~Joachims, editors, \emph{27th Int. Conf.
  Mach. Learn.}, Haifa, Israel, 2010{\natexlab{a}}.

\bibitem[Kolar et~al.(2010{\natexlab{b}})Kolar, Song, Ahmed, and
  Xing]{Kolar2010Estimating}
M.~Kolar, L.~Song, A.~Ahmed, and E.~P. Xing.
\newblock Estimating {Time-varying} networks.
\newblock \emph{Ann. Appl. Stat.}, 4\penalty0 (1):\penalty0 94--123,
  2010{\natexlab{b}}.

\bibitem[Kolar et~al.(2013)Kolar, Liu, and Xing]{kolar13multiatticml}
M.~Kolar, H.~Liu, and E.~P. Xing.
\newblock Markov network estimation from multi-attribute data.
\newblock In \emph{ICML}, 2013.

\bibitem[Kolar et~al.(2014)Kolar, Liu, and Xing]{Kolar2014Graph}
M.~Kolar, H.~Liu, and E.~P. Xing.
\newblock Graph estimation from multi-attribute data.
\newblock \emph{J. Mach. Learn. Res.}, 15\penalty0 (1):\penalty0 1713--1750,
  2014.

\bibitem[Kuang et~al.(2016{\natexlab{a}})Kuang, Thomson, Caldwell, Peissig,
  Stewart, and Page]{kuang2016baseline}
Z.~Kuang, J.~Thomson, M.~Caldwell, P.~Peissig, R.~Stewart, and D.~Page.
\newblock Baseline regularization for computational drug repositioning with
  longitudinal observational data.
\newblock In \emph{IJCAI}, 2016{\natexlab{a}}.

\bibitem[Kuang et~al.(2016{\natexlab{b}})Kuang, Thomson, Caldwell, Peissig,
  Stewart, and Page]{kuang2016computational}
Z.~Kuang, J.~Thomson, M.~Caldwell, P.~Peissig, R.~Stewart, and D.~Page.
\newblock Computational drug repositioning using continuous self-controlled
  case series.
\newblock In \emph{SIGKDD}, 2016{\natexlab{b}}.

\bibitem[Kuang et~al.(2017{\natexlab{a}})Kuang, Geng, and
  Page]{kuang2017screening}
Z.~Kuang, S.~Geng, and D.~Page.
\newblock A screening rule for l1-regularized ising model estimation.
\newblock In \emph{NIPS}, 2017{\natexlab{a}}.

\bibitem[Kuang et~al.(2017{\natexlab{b}})Kuang, Peissig, Santos~Costa, Maclin,
  and Page]{kuang2017pharmacovigilance}
Z.~Kuang, P.~Peissig, V.~Santos~Costa, R.~Maclin, and D.~Page.
\newblock Pharmacovigilance via baseline regularization with large-scale
  longitudinal observational data.
\newblock In \emph{SIGKDD}, 2017{\natexlab{b}}.

\bibitem[Laumann et~al.(2016)Laumann, Snyder, Mitra, Gordon, Gratton, Adeyemo,
  Gilmore, Nelson, Berg, Greene, et~al.]{laumann2016stability}
T.~O. Laumann, A.~Z. Snyder, A.~Mitra, E.~M. Gordon, C.~Gratton, B.~Adeyemo,
  A.~W. Gilmore, S.~M. Nelson, J.~J. Berg, D.~J. Greene, et~al.
\newblock On the stability of bold fmri correlations.
\newblock \emph{Cerebral cortex}, 27\penalty0 (10):\penalty0 4719--4732, 2016.

\bibitem[Lauritzen(1996)]{Lauritzen1996Graphical}
S.~L. Lauritzen.
\newblock \emph{Graphical Models}, volume~17 of \emph{Oxford Statistical
  Science Series}.
\newblock The Clarendon Press Oxford University Press, New York, 1996.
\newblock Oxford Science Publications.

\bibitem[Lee and Hastie(2015)]{lee2015learning}
J.~D. Lee and T.~J. Hastie.
\newblock Learning the structure of mixed graphical models.
\newblock \emph{J. Comput. Graph. Statist.}, 24\penalty0 (1):\penalty0
  230--253, 2015.

\bibitem[Lee and Liu(2015)]{lee2015joint}
W.~Lee and Y.~Liu.
\newblock Joint estimation of multiple precision matrices with common
  structures.
\newblock \emph{J. Mach. Learn. Res.}, 16:\penalty0 1035--1062, 2015.

\bibitem[Li et~al.(2013)Li, Lin, and Racine]{li2013optimal}
Q.~Li, J.~Lin, and J.~S. Racine.
\newblock Optimal bandwidth selection for nonparametric conditional
  distribution and quantile functions.
\newblock \emph{J. Bus. Econom. Statist.}, 31\penalty0 (1):\penalty0 57--65,
  2013.

\bibitem[Liu et~al.(2014)Liu, Zhang, Burnside, and Page]{liu2014multiple}
J.~Liu, C.~Zhang, E.~S. Burnside, and D.~Page.
\newblock Multiple testing under dependence via semiparametric graphical
  models.
\newblock In \emph{ICML}, pages 955--963, 2014.

\bibitem[Liu et~al.(2016)Liu, Zhang, and Page]{liu2016multiple}
J.~Liu, C.~Zhang, and D.~Page.
\newblock Multiple testing under dependence via graphical models.
\newblock \emph{Ann. Appl. Stat.}, 10\penalty0 (3):\penalty0 1699--1724, 2016.

\bibitem[Lu et~al.(2018)Lu, Kolar, and Liu]{Lu2015Posta}
J.~Lu, M.~Kolar, and H.~Liu.
\newblock Post-regularization inference for time-varying nonparanormal
  graphical models.
\newblock \emph{J. Mach. Learn. Res.}, 18\penalty0 (203):\penalty0 1--78, 2018.

\bibitem[Mignotte et~al.(2000)Mignotte, Collet, Perez, and
  Bouthemy]{mignotte2000sonar}
M.~Mignotte, C.~Collet, P.~Perez, and P.~Bouthemy.
\newblock Sonar image segmentation using an unsupervised hierarchical mrf
  model.
\newblock \emph{IEEE transactions on image processing}, 9\penalty0
  (7):\penalty0 1216--1231, 2000.

\bibitem[Mohan et~al.(2014)Mohan, London, Fazel, Witten, and
  Lee]{Mohan2014Node}
K.~Mohan, P.~London, M.~Fazel, D.~M. Witten, and S.-I. Lee.
\newblock Node-based learning of multiple gaussian graphical models.
\newblock \emph{J. Mach. Learn. Res.}, 15:\penalty0 445--488, 2014.

\bibitem[Poldrack et~al.(2015)Poldrack, Laumann, Koyejo, Gregory, Hover, Chen,
  Gorgolewski, Luci, Joo, Boyd, et~al.]{poldrack2015long}
R.~A. Poldrack, T.~O. Laumann, O.~Koyejo, B.~Gregory, A.~Hover, M.-Y. Chen,
  K.~J. Gorgolewski, J.~Luci, S.~J. Joo, R.~L. Boyd, et~al.
\newblock Long-term neural and physiological phenotyping of a single human.
\newblock \emph{Nature communications}, 6:\penalty0 8885, 2015.

\bibitem[Power et~al.(2014)Power, Mitra, Laumann, Snyder, Schlaggar, and
  Petersen]{power2014methods}
J.~D. Power, A.~Mitra, T.~O. Laumann, A.~Z. Snyder, B.~L. Schlaggar, and S.~E.
  Petersen.
\newblock Methods to detect, characterize, and remove motion artifact in
  resting state {fMRI}.
\newblock \emph{{NeuroImage}}, 84:\penalty0 320--341, 2014.

\bibitem[Price et~al.(2014)Price, Wee, Gao, and Shen]{price2014multiple}
T.~Price, C.-Y. Wee, W.~Gao, and D.~Shen.
\newblock Multiple-network classification of childhood autism using functional
  connectivity dynamics.
\newblock In \emph{International Conference on Medical Image Computing and
  Computer-Assisted Intervention}, pages 177--184. Springer, 2014.

\bibitem[Ren et~al.(2015)Ren, Sun, Zhang, and Zhou]{Ren2013Asymptotic}
Z.~Ren, T.~Sun, C.-H. Zhang, and H.~H. Zhou.
\newblock Asymptotic normality and optimalities in estimation of large
  {G}aussian graphical models.
\newblock \emph{Ann. Stat.}, 43\penalty0 (3):\penalty0 991--1026, 2015.

\bibitem[Sadaghiani and Kleinschmidt(2013)]{sadaghiani2013functional}
S.~Sadaghiani and A.~Kleinschmidt.
\newblock Functional interactions between intrinsic brain activity and
  behavior.
\newblock \emph{Neuroimage}, 80:\penalty0 379--386, 2013.

\bibitem[Sheffield et~al.(2015)Sheffield, Repovs, Harms, Carter, Gold,
  MacDonald~III, Ragland, Silverstein, Godwin, and Barch]{sheffield2015fronto}
J.~M. Sheffield, G.~Repovs, M.~P. Harms, C.~S. Carter, J.~M. Gold, A.~W.
  MacDonald~III, J.~D. Ragland, S.~M. Silverstein, D.~Godwin, and D.~M. Barch.
\newblock Fronto-parietal and cingulo-opercular network integrity and cognition
  in health and schizophrenia.
\newblock \emph{Neuropsychologia}, 73:\penalty0 82--93, 2015.

\bibitem[Shine et~al.(2015)Shine, Koyejo, Bell, Gorgolewski, Gilat, and
  Poldrack]{shine2015estimation}
J.~M. Shine, O.~Koyejo, P.~T. Bell, K.~J. Gorgolewski, M.~Gilat, and R.~A.
  Poldrack.
\newblock Estimation of dynamic functional connectivity using multiplication of
  temporal derivatives.
\newblock \emph{NeuroImage}, 122:\penalty0 399--407, 2015.

\bibitem[Shine et~al.(2016)Shine, Koyejo, and Poldrack]{shine2016temporal}
J.~M. Shine, O.~Koyejo, and R.~A. Poldrack.
\newblock Temporal metastates are associated with differential patterns of
  time-resolved connectivity, network topology, and attention.
\newblock \emph{PNAS}, 113\penalty0 (35):\penalty0 9888--9891, 2016.

\bibitem[Song et~al.(2009{\natexlab{a}})Song, Kolar, and Xing]{le09keller}
L.~Song, M.~Kolar, and E.~P. Xing.
\newblock Keller: Estimating time-varying interactions between genes.
\newblock \emph{Bioinformatics}, 25\penalty0 (12):\penalty0 i128--i136,
  2009{\natexlab{a}}.

\bibitem[Song et~al.(2009{\natexlab{b}})Song, Kolar, and Xing]{song09time}
L.~Song, M.~Kolar, and E.~P. Xing.
\newblock Time-varying dynamic bayesian networks.
\newblock In Y.~Bengio, D.~Schuurmans, J.~D. Lafferty, C.~K.~I. Williams, and
  A.~Culotta, editors, \emph{NIPS}, pages 1732--1740, 2009{\natexlab{b}}.

\bibitem[Suggala et~al.(2017)Suggala, Kolar, and
  Ravikumar]{Suggala2017Expxorcist}
A.~S. Suggala, M.~Kolar, and P.~Ravikumar.
\newblock {The Expxorcist}: Nonparametric graphical models via conditional
  exponential densities.
\newblock In \emph{NIPS 30}, 2017.

\bibitem[Sun et~al.(2015)Sun, Kolar, and Xu]{Sun2015Learning}
S.~Sun, M.~Kolar, and J.~Xu.
\newblock Learning structured densities via infinite dimensional exponential
  families.
\newblock In \emph{NIPS 28}, 2015.

\bibitem[Tohid et~al.(2015)Tohid, Faizan, and Faizan]{tohid2015alterations}
H.~Tohid, M.~Faizan, and U.~Faizan.
\newblock Alterations of the occipital lobe in schizophrenia.
\newblock \emph{Neurosciences}, 20\penalty0 (3):\penalty0 213, 2015.

\bibitem[Van~Dijk et~al.(2012)Van~Dijk, Sabuncu, and Buckner]{van2012influence}
K.~R. Van~Dijk, M.~R. Sabuncu, and R.~L. Buckner.
\newblock The influence of head motion on intrinsic functional connectivity
  mri.
\newblock \emph{Neuroimage}, 59\penalty0 (1):\penalty0 431--438, 2012.

\bibitem[Varoquaux et~al.(2010)Varoquaux, Gramfort, Poline, and
  Thirion]{varoquaux2010brain}
G.~Varoquaux, A.~Gramfort, J.-B. Poline, and B.~Thirion.
\newblock Brain covariance selection: Better individual functional connectivity
  models using population prior.
\newblock In \emph{NIPS}, 2010.

\bibitem[Wang and Kolar(2014)]{wang2014inference}
J.~Wang and M.~Kolar.
\newblock Inference for sparse conditional precision matrices.
\newblock \emph{arXiv:1412.7638}, 2014.

\bibitem[Wang and Kolar(2016)]{Wang2016Inference}
J.~Wang and M.~Kolar.
\newblock Inference for high-dimensional exponential family graphical models.
\newblock In \emph{AISTATS}, 2016.

\bibitem[Wasserman et~al.(2014)Wasserman, Kolar, and
  Rinaldo]{Wasserman2014Berry}
L.~A. Wasserman, M.~Kolar, and A.~Rinaldo.
\newblock Berry-{E}sseen bounds for estimating undirected graphs.
\newblock \emph{Electron. J. Stat.}, 8:\penalty0 1188--1224, 2014.

\bibitem[Yin et~al.(2010)Yin, Geng, Li, and Wang]{yin10nonparametric}
J.~Yin, Z.~Geng, R.~Li, and H.~Wang.
\newblock Nonparametric covariance model.
\newblock \emph{Stat. Sinica}, 20:\penalty0 469--479, 2010.

\bibitem[Yu et~al.(2016)Yu, Gupta, and Kolar]{Yu2016Statistical}
M.~Yu, V.~Gupta, and M.~Kolar.
\newblock Statistical inference for pairwise graphical models using score
  matching.
\newblock In \emph{NIPS 29}, 2016.

\bibitem[Yu et~al.(2019)Yu, Gupta, and Kolar]{Yu2019Simultaneous}
M.~Yu, V.~Gupta, and M.~Kolar.
\newblock Simultaneous inference for pairwise graphical models with generalized
  score matching.
\newblock \emph{arXiv 1905.06261}, 2019.

\bibitem[Zalesky et~al.(2012)Zalesky, Fornito, and Bullmore]{zalesky2012use}
A.~Zalesky, A.~Fornito, and E.~Bullmore.
\newblock On the use of correlation as a measure of network connectivity.
\newblock \emph{Neuroimage}, 60\penalty0 (4):\penalty0 2096--2106, 2012.

\bibitem[Zhou et~al.(2010)Zhou, Lafferty, and Wasserman]{Zhou08time}
S.~Zhou, J.~D. Lafferty, and L.~A. Wasserman.
\newblock Time varying undirected graphs.
\newblock \emph{Mach. Learn.}, 80\penalty0 (2-3):\penalty0 295--319, 2010.

\end{thebibliography}
